%% file: main.tex
\documentclass{article} 
\usepackage{iclr2026_conference,times}
\usepackage{array}      
\input{math_commands.tex}
\newcommand{\nb}[1]{\textcolor{blue}{[nb: #1]}}

\usepackage{url}
\usepackage{acronym}
\usepackage{mathtools}
\usepackage{subcaption}
\usepackage{booktabs}
\usepackage{multirow}
\usepackage{enumitem}
\usepackage{xspace}
\usepackage{thm-restate}
\usepackage{wrapfig}
\usepackage[ruled,vlined]{algorithm2e}

\usepackage[final, colorlinks=true, allcolors=blue, pagebackref=true]{hyperref}

\usepackage[capitalize,noabbrev]{cleveref}
\usepackage[createShortEnv,conf={no link to proof, text link},commandRef=autoref]{proof-at-the-end}

\renewcommand*{\backrefalt}[4]{%
    \ifcase #1 \footnotesize{(Not cited.)}%
    \or        \footnotesize{(Cited on page~#2)}%
    \else      \footnotesize{(Cited on pages~#2)}%
    \fi}

\newcommand{\ie}{\emph{i.e.},~}
\newcommand{\eg}{\emph{e.g.},~}

\hypersetup{
	colorlinks=true,       
	linkcolor=blue,        
	citecolor=blue,        
	filecolor=magenta,     
	urlcolor=blue         
}

\acrodef{RFM}{Riemannian Flow Matching}

\title{Generalised Flow Maps for Few-Step Generative Modelling on Riemannian Manifolds}

\author{Oscar Davis\textsuperscript{\normalfont 1}\thanks{Correspondence to \texttt{oscar.davis@cs.ox.ac.uk}}~{\normalfont,~} Michael S. Albergo\textsuperscript{\normalfont 2,3,4}{\normalfont,~} Nicholas M. Boffi\textsuperscript{\normalfont 5}, 
Michael M. Bronstein\textsuperscript{\normalfont 1,6}{\normalfont,~}\\
\textbf{Avishek Joey Bose}\textsuperscript{\normalfont 1,7,8}\\
\textsuperscript{\normalfont 1}University of Oxford, \textsuperscript{\normalfont 2}Harvard University, \textsuperscript{\normalfont 3}Kempner Institute,\\\textsuperscript{\normalfont 4}Institute for Artificial Intelligence and Fundamental Interactions, MIT, \\\textsuperscript{\normalfont 5}Carnegie Mellon University, \textsuperscript{\normalfont 6}AITHYRA, \textsuperscript{\normalfont 7}Mila, \textsuperscript{\normalfont 8}Imperial College London
}

\usepackage{amsthm}
\newtheorem{proposition}{Proposition}
\newtheorem{definition}{Definition}
\newtheorem{lemma}{Lemma}

\iclrfinalcopy 
\begin{document}

\maketitle

\begin{abstract}
\looseness=-1
Geometric data 
and purpose-built generative models on them have become ubiquitous in high-impact deep learning application domains, ranging from protein backbone generation and computational chemistry to geospatial data. 
Current geometric generative models remain computationally expensive at inference---requiring many steps of complex numerical simulation---as they are derived from dynamical measure transport frameworks such as diffusion and flow-matching on Riemannian manifolds. In this paper, we propose \namelong (\nameshort), a new class of few-step generative models that generalises the Flow Map framework in Euclidean spaces
to arbitrary Riemannian manifolds. 
We instantiate \nameshorts with three self-distillation-based training methods: Generalised Lagrangian Flow Maps, Generalised Eulerian Flow Maps, and Generalised Progressive Flow Maps. We theoretically show that \nameshorts, under specific design decisions, unify and elevate existing Euclidean few-step generative models, such as consistency models, shortcut models, and meanflows, to the Riemannian setting. We benchmark \nameshorts against other geometric generative models on a suite of geometric datasets, including geospatial data, RNA torsion angles, and hyperbolic manifolds, and   
achieve state-of-the-art sample quality for single- and few-step evaluations, and superior or competitive log-likelihoods using the implicit probability flow.


 \cut{We further demonstrate the flexibility of \nameshorts by introducing a novel parametrization for Lie groups that guarantees manifold constraints without ambient-space projection by exploiting the Lie algebra.}

\end{abstract}

\input{section/introduction.tex}
\input{section/background.tex}
\input{section/method.tex}

\input{section/experiments.tex}
\input{section/related_work.tex}
\input{section/conclusion.tex}

\input{section/statements.tex}

\clearpage

\bibliography{references}
\bibliographystyle{iclr2026_conference}

\clearpage
\input{section/appendix.tex}

\end{document}

%% file: math_commands.tex
\usepackage{amsmath,amsfonts,bm}

\newcommand{\calL}{\mathcal{L}}

\def\eqref#1{equation~\ref{#1}}

\def\1{\bm{1}}

\DeclareMathAlphabet{\mathsfit}{\encodingdefault}{\sfdefault}{m}{sl}
\SetMathAlphabet{\mathsfit}{bold}{\encodingdefault}{\sfdefault}{bx}{n}

\def\gL{{\mathcal{L}}}
\def\gM{{\mathcal{M}}}

\def\gP{{\mathcal{P}}}

\def\gT{{\mathcal{T}}}
\def\gU{{\mathcal{U}}}

\newcommand{\E}{\mathbb{E}}

\newcommand{\R}{\mathbb{R}}

\newcommand{\KL}{\mathbb{D}_{\mathrm{KL}}}

\newcommand{\xhdr}[1]{{\noindent\bfseries #1}.}
\newcommand{\cut}[1]{}

\newcommand{\namelong}{\textsc{Generalised Flow Maps}\xspace}
\newcommand{\nameshort}{\textsc{GFM}\xspace}
\newcommand{\nameshorts}{\textsc{GFM}\textnormal{s}\xspace}

\newcommand{\sethree}{\mathrm{SE(3)}}

\newcommand{\sothree}{\mathrm{SO(3)}}
\newcommand{\sethreen}{\sethree^{\scriptscriptstyle N}}

\newcommand{\ddd}{\mathrm{d}}

\newcommand{\diff}{\mathrm{d}}
\newcommand{\sg}{\mathrm{stopgrad}}

\newcommand{\MMD}{\mathrm{MMD}}

\newcommand{\oursesd}{\text{\scshape G-ESD}\xspace}
\newcommand{\ourspsd}{\text{\scshape G-PSD}\xspace}
\newcommand{\ourslsd}{\text{\scshape G-LSD}\xspace}
\newcommand{\oursmf}{\text{\scshape G-MF}\xspace}

\newcommand{\M}{\mathcal{M}}
\newcommand{\T}{\mathcal{T}}
\newcommand{\proj}{\mathrm{proj}}

\newcommand{\lsd}{\mathcal{L}_\mathsf{LSD}(\hat{v})}
\newcommand{\esd}{\mathcal{L}_\mathsf{ESD}(\hat{v})}
\newcommand{\psd}{\mathcal{L}_\mathsf{PSD}(\hat{v})}

\newcommand{\sorth}{\ensuremath{\mathrm{SO}(3)}}

\newcommand\joey[1]{\noindent{\color{orange} {\bf \fbox{Joey}} {\it#1}}}
\usepackage[framemethod=TikZ]{mdframed}
\mdfdefinestyle{MyFrame}{%
    linecolor=black,
    outerlinewidth=.3pt,
    roundcorner=5pt,
    innertopmargin=1pt, 
    innerbottommargin=1pt, 
    innerrightmargin=1pt,
    innerleftmargin=1pt,
    backgroundcolor=black!0!white}

\mdfdefinestyle{MyFrame2}{%
    linecolor=white,
    outerlinewidth=1pt,
    roundcorner=1pt,
    innertopmargin=3pt,
    innerbottommargin=2pt,
    innerrightmargin=7pt,
    innerleftmargin=7pt,
    backgroundcolor=black!3!white}

\mdfdefinestyle{MyFrameEq}{%
    linecolor=white,
    outerlinewidth=0pt,
    roundcorner=0pt,
    innertopmargin=0pt,
    innerbottommargin=0pt,
    innerrightmargin=7pt,
    innerleftmargin=7pt,
    backgroundcolor=black!3!white}

%% file: section/introduction.tex
\section{Introduction}
\label{sec:introduction}

\looseness=-1
Dynamical measure transport offers a unifying and prescriptive framework for constructing neural network-based generative models that learn to sample a desired target distribution by pushing forward a tractable prior. 
Numerical solution of the resulting dynamical systems has led to popular method families---including diffusion models~\citep{song2020generativemodelingestimatinggradients}, flow matching~\citep{lipman2022flow,peluchetti2023non,liu2022flow, albergo2022building}, and stochastic interpolants~\citep{stochasticinterpolants}---which together have revolutionised the field and led to state-of-the-art results over continuous modalities~\citep{karras2024analyzing,polyak2024movie}. 
While often applied to Euclidean data such as images, this powerful paradigm naturally extends to data types that are inherently \emph{geometric} and lie on a known Riemannian manifold, in which case the associated flows and diffusions are defined directly on the manifold~\citep{huang2022riemannian,de2022riemannian,rfm}.
Such geometric generative models have found widespread application in high-impact scientific settings such as rational drug design on the $\sethreen$ manifold of protein backbones~\citep{se3sfm,huguet2024sequence,watson2023novo}, generative material design in computational chemistry~\citep{miller2024flowmm}, and even discrete data using the Fisher-Rao geometry on the probability simplex~\citep{fisher,categoricalfm}.

\looseness=-1
The scalability of dynamical transport-based generative models arises from their use of simple regression-based objectives that lead to rapid \emph{simulation-free} training.
However, unlike training, obtaining high-quality generated samples at inference time requires numerical solution of the parametrised dynamical system, which necessitates evaluating the large learnt model numerous times.
The computational complexity of inference is further burdened in geometric settings, where each step of simulation requires computing potentially numerically unstable operators that must guarantee manifold constraints for faithful inference. 
For instance, for matrix Lie groups like $\sothree$, computing the standard exponential and logarithmic maps requires truncating an infinite matrix power series~\citep{al2012improved}, which may incur additional sources of error beyond just step-size discretisation error compared to inference on Euclidean spaces. Therefore, reducing the number of Riemannian operations, while maintaining this powerful inductive bias, has the potential to lead to higher fidelity generated samples.

\begin{wrapfigure}[16]{r}{0.5\textwidth}
  \centering
  \vspace{-4mm}
  \includegraphics[width=\linewidth]{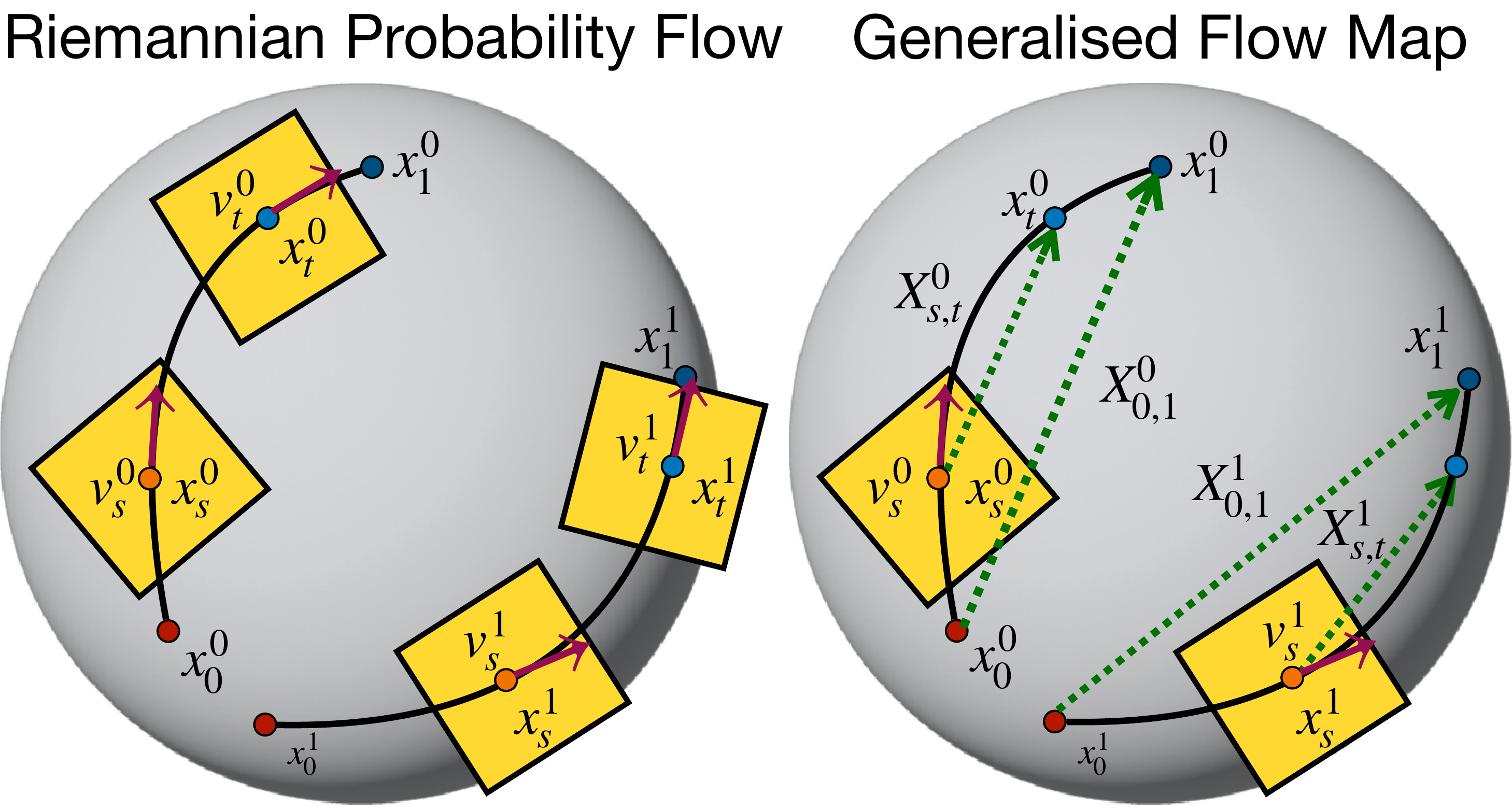}
  \caption{\small \looseness=-1 \textbf{Left:} The instantaneous vector field $v_t$ of the Riemannian probability flow ODE.  \textbf{Right:} The Generalised flow map $X_{s,t}$ that jumps along the trajectory of the Riemannian probability flow ODE by steps of size $t-s$. Both models are depicted on $\mathbb{S}^2$.}
  \label{fig:intro_fig}
\end{wrapfigure}

\looseness=-1
The search for accelerated inference with dynamical generative models defined on Euclidean spaces has spurred the development of fast inference techniques that fall broadly into two methodological families.
The first family involves the intricate design of higher-order samplers~\citep{dockhorn2022genie,zhang2022fast,sabour2024align,karras2022elucidating}, which can produce high-quality samples with less than ten steps.
The second pursues direct learning of the flow map---the global solution of the governing differential equation---rather than the instantaneous vector field used in standard simulation~\citep{boffi2024flow}. 
This flow map perspective underlies distillation-based approaches, the most popular of which include consistency models~\citep{song2023consistency,song2023improved,lu2024simplifying}, which attempt to learn the one-step flow map of a pre-trained teacher.
Distillation-based methods have also been extended to the multi-step flow map setting we consider here at scale~\citep{ayf}.
Recent advances have further demonstrated that state-of-the-art one-step generative models can be trained from scratch without a pre-trained teacher through self-consistent \textit{self-distillation}~\citep{flowmaps,shortcut} or Monte Carlo direct training~\citep{meanflows} objectives.
These exciting developments naturally raise the question: 
\begin{center}
\vspace{-5pt}
    \textit{Can the same flow map principles be extended to Riemannian settings, leading to a new class of few-step geometric generative models?}
\vspace{-5pt}
\end{center}
Here, we answer this question in the affirmative by introducing \namelong (\nameshort), a method that generalises and extends flow maps~\citep{flowmaps} to arbitrary Riemannian manifolds.
Overall, our main contributions can be summarised as:

\begin{enumerate}[topsep=0pt, partopsep=0pt, itemsep=0pt, parsep=0pt, leftmargin=*]
    \item \textbf{Models}. 
    We derive the technical conditions required for a flow map to respect manifold constraints, showing how they lead to three variants of \nameshort: Eulerian, Lagrangian, and Progressive.
    \item \textbf{Learning}.
    We elevate the self-distillation losses presented in~\citet{flowmaps} to the Riemannian setting, showing how each variant can be learned efficiently, and how each recovers their Euclidean counterparts as a specific case.
   \cut{ %
    \item \textbf{Efficiency.} 
    We enhance the flexibility of \nameshort over matrix Lie groups, offering a novel parametrization that learns the map on the more numerically favorable Lie algebra.}
    %
    \item \textbf{Empirical validation}. 
    We comprehensively demonstrate the performance benefits of \nameshort over previous Riemannian generative models on standard geometric datasets. 
    Specifically, we obtain state-of-the-art performance for single- and few-step sample quality on torsion angles found in RNA backbones and protein side-chains, geospatial data of natural disasters on the Earth's surface, synthetic data on $\sothree$ and on the Poincar\'e disk model of hyperbolic geometry, demonstrating our method's robustness to manifolds with non-trivial structures.
\end{enumerate}
Our methodology opens the door to an under-explored design space of highly expressive generative models defined over geometric data that enjoy rapid inference.
%
%
%

\cut{
We present three flavours of self-distillation losses (\ie, losses that use older parameters as teacher models~\citep{consistency,sCM,ayf,meanflows,imm}) that recover the three methods proposed by flow maps on Euclidean spaces. Moreover, our experiments demonstrate the capabilities of our method on a host of datasets supported on Riemannian manifolds, and exhibit state-of-the-art performance for single- and few-step sample quality while achieving competitive negative log-likelihood values.
}

%% file: section/background.tex
\section{Background}
\label{sec:background}

\subsection{Riemannian geometry}
\looseness=-1
Informally, a $d$-dimensional manifold $\gM$ is a topological space covered by a family of sets and corresponding maps, called \emph{charts}, that render the manifold locally homeomorphic to $\R^d$. Requiring a $C^{\infty}$ differential structure on the charts makes the manifold \emph{smooth}, and thereby naturally defines the tangency of a vector $v \in \T_x \M$ to a point $x \in \M$, where $\gT_x \M$ denotes the tangent space at point $x$, and the disjoint union of all tangent spaces $\mathcal{T}\mathcal{M} \coloneqq \bigsqcup_{x \in \mathcal{M}} \mathcal{T}_x \mathcal{M}$ forms the \emph{tangent bundle}.  

\looseness=-1
A \emph{Riemannian manifold} is a pair $(\mathcal{M}, g)$, where $\M$ is a smooth manifold, and $g$ is a
\emph{metric tensor}, which smoothly assigns a bilinear and symmetric positive semi-definite inner product $
g_x : \mathcal{T}_x \mathcal{M} \times \mathcal{T}_x \mathcal{M} \to \mathbb{R}$. This structure induces notions of curve lengths, 
distances, volumes, and angles. We assume $\mathcal{M}$ is connected,
orientable, and complete, with volume element $\ddd x$. A \emph{vector field} is a function 
from $\M$ to $\T\M$, such that $x\mapsto v(x) \in \T_x\M$. For a continuously differentiable map, $f:\M\to\M$, we define its \emph{differential} at $x \in \M$, $\ddd f_x:\T_x\M\to \T_{f(x)}\M$, as the map given by $\ddd f_x:v\mapsto (f \circ \gamma)'(0)$ for \emph{any} continuously differentiable curve $\gamma:[0,1]\to\M$, such that $\gamma(0) = x$ and $\gamma'(0) = v$. The resulting map between tangent spaces is independent of $\gamma$.

\cut{
Each $\T_p\M$ being a vector space, endowing it with an inner product, $g_p$, makes $\T_p\M$ an inner product space. If $g_p$ is positive semi-definite, symmetric, and bilinear, and varies smoothly with $p$, then the pair $(\M, g)$ is said to be a Riemannian manifold, and $g$ is a metric. This metric, in turn, defines, on the manifold, the distance between two points, angles, and lengths of vectors. The exponential map, $\exp_p:\T_p\M\to\M$, generalises $p + v$, and amounts to taking a step in the direction of the vector $v$. The logarithmic map, $\log_p:\M\to\T_p\M$, generalises $q-p$, and provides a vector pointing towards its argument. The shortest path/geodesic between two points $p,q \in \gM$ is the curve $\gamma:[0,1]\to \gM$, with $\gamma(0)=p$ and $\gamma(1)=q$, that minimises the energy functional, $\gamma\mapsto  \frac 12 \int_0^1 \left\lVert\dot{\gamma}(t)\right\lVert_{\gamma(t)}^2 \diff t$. }

\looseness=-1
\xhdr{Manifold operations} In the context of this work, we require a few usual operations on Riemannian manifolds. Specifically, we require a tangential projection from the ambient space, $\R^d$, to any arbitrary tangent plane $\T_x\M$ for $x\in\M$. The tangential projection enjoys a closed form: $P_x = I - \frac{n_x n_x^\top}{\|n_x\|^2}$, where $n_x \in \gT_x\gM$ is a basis vector.
Roughly, the exponential map, $\exp_x:\T_x\M\to\M$, generalises the Euclidean notion of ``$x + v$'', and amounts to taking a step in the direction of the vector $v$. 
The logarithmic map, $\log_x:\M\to\T_x\M$, generalises ``$q-x$'' and provides a vector pointing towards its argument. Finally, the shortest path between two points $p,q \in \gM$ is called a \emph{geodesic} and is the curve $\gamma:[0,1]\to \gM$, with $\gamma(0)=p$ and $\gamma(1)=q$. 

\subsection{Stochastic Interpolants and Flow Maps}
\label{sec:background_flow_maps}

\looseness=-1
We are interested in the dynamical measure transport problem, which seeks to transport an easy-to-sample reference measure $\rho_0$ to a specified target measure $\rho_1$. As Riemannian manifolds can be embedded within a larger ambient space, $\gM \subseteq \R^d$, we consider both $\rho_0, \rho_1 \in \gP(\gM) \subset \gP(\R^d)$. 

%
\looseness=-1
\xhdr{Stochastic interpolants in $\R^d$}
In the context of generative modelling, a target distribution of interest is provided in the form of a dataset of samples, $\mathcal D = \{x_1^i\}_{i=1}^n$, where $x_1^i \overset{\text{iid}}{\sim} \rho_1(x_1)$, thus defining an empirical distribution, $\rho_{\text{data}}(x_1) \coloneqq \frac 1 n \sum^n_i \delta (x_1 - x^i_1)$.
The goal in generative modelling is to build a model $\rho^{\theta}_{1}$, with parameters $\theta \in \Theta \subset \R^m$, that approximates $\rho^{\theta}_1 \approx \rho_1$ in the distributional sense (\eg $W_2(\rho_1, \rho^{\theta}_1)$, or $\KL(\rho_1\mid \mid \rho^{\theta}_1)$) while allowing us to draw new samples.

\looseness=-1
A modern and scalable approach for solving the dynamical measure transport problem---and by extension the generative modelling problem---is to leverage the framework of stochastic interpolants~\citep{stochasticinterpolants}. A stochastic interpolant is a stochastic process $I:[0,1]\times\R^d\times\R^d\to\R^d$ that combines, in time, samples from the reference and target measures $(t,x_0,x_1)\mapsto I_{t}(x_0, x_1) = \alpha_t x_0 + \beta_t x_1$, for some choice of continuously differentiable functions $\alpha$ and $\beta$, such that $\alpha_0 = 1$, $\alpha_1 = 0$, $\beta_0 = 0$ and $\beta_1 = 1$. In addition, drawing pairs $(x_0, x_1) \sim \rho(x_0, x_1)$ from a given coupling must further satisfy the following marginal constraints $\int \rho(x_0, x_1)\ddd x_0 = \rho(x_1)$ and $\int \rho(x_0, x_1)\ddd x_1 = \rho(x_0)$. Connecting the measures $\rho_0$ and $\rho_1$, the interpolant defines a probability path, $(\rho_t)_{t\in[0,1]}$, which also follows the corresponding probability flow and transports particles using the following ordinary differential equation (ODE):
\begin{align}
    \partial_t{x}_t &= \E_{x_t \sim \rho_t(x_t)}\left[\partial_t{I}_{t} \mid I_{t} = x_t\right] = v_{t}(x_t), \quad \text{with }x_0 \sim \rho_0(x_0).
    \label{eq:probflow}
\end{align}
\looseness=-1
Setting $\alpha_t \equiv 1-t$ and $\beta_t \equiv t$ recovers well known methods such as flow-matching~\citep{albergo2022building,lipman2022flow} and rectified flow~\citep{liu2022flow}, while imposing optimal transport costs on the coupling results in OT-CFM~\citep{tong2023improving}.

\looseness=-1
\xhdr{Flow maps in $\R^d$} An alternative path to generating samples is to instead learn the integrator of the probability flow ODE directly, in order to avoid, at inference time, the costly numerical integration. Introduced in~\citet{boffi2024flow}, flow maps are functions $X:[0,1]^2\times\R^d\to\R^d$ that satisfy the \emph{jump condition}: $X_{s,t}(x_s) = x_t$, where $(x_t)_{t\in[0,1]}$ is any solution of the probability flow. One can thus sample from $\rho_1$ by first sampling $x_0\sim\rho_0$ and then applying $X_{0,1}(x_0)\sim \rho_1$. These are parametrised as $X^{\theta}_{s,t}(x_s) = x + (t-s) v^{\theta}_{s,t}(x_s)$, for any $x$, $s$ and $t$, so that the boundary condition, at $t = s$, $X^{\theta}_{t,t} = \mathrm{Id}$ is automatically satisfied. There are various self-distillation (\ie from scratch) objectives for learning a flow-map, which are used in conjunction with typical flow-matching objectives~\citep{stochasticinterpolants,tong2023improving,flowmatching},
For brevity, the three (Euclidean) flow map self-distillation loss variations---Eulerian, Lagrangian, and progressive---are provided in~\S\ref{app:flow_map_losses_original}.

\cut{
\xhdr{Riemannian manifolds} We note that dynamical flow-based sampling has already been properly defined on Riemannian manifolds in~\citet{rfm} and \cref{eq:probflow} extends naturally as long as $\rho_0$ and $\rho_1$ are supported on $(\M, g)$, and $v:\M\times[0,1]\to \T\M$. \joey{I would remove this. This is also not entirely correct. There was work (long) before this on flow-based models on manifolds.}

}

%% file: section/method.tex
\section{Method}

\looseness=-1
We seek to define a generative model over arbitrary Riemannian manifolds that enables accelerated inference through few-step sampling. Towards this goal, we generalise the notion of flow maps to Riemannian manifolds, yielding \namelong (\nameshort) in~\S\ref{sec:gfm_main}. We then show how to train such \nameshort from scratch using self-distillation losses in~\S\ref{sec:riemannian_self_distillation}, defined for arbitrary Riemannian manifolds---recovering the Euclidean case of~\citet{flowmaps} as a special case. 



\subsection{Generalised Flow Maps}
\label{sec:gfm_main}

\looseness=-1
We begin by defining a flow map in the context of Riemannian manifolds $\gM$. To do so, we first recall that the flow map allows us to jump along the trajectory of the probability flow ODE connecting two measures $\rho_0, \rho_1 \in \gP(\gM)$.
The transport described by this ODE can be written in terms of the corresponding Riemannian flow and continuity equations:
\begin{align}
    \begin{array}{rl}
        \partial_t x_t  & = v_t( x_t)\\
    \partial_t \rho_t (x) & = - \mathrm{div}_g \left( \rho_t (x)v_t(x) \right),
    \end{array}
    \label{eq:prob_flow_on_manifolds}
\end{align}
\looseness=-1
with $x_0\sim \rho_0$, and where $\mathrm{div}_g$ denotes the Riemannian divergence induced by the metric $g$. As the probability flow ODE lives on  $\gM$, this presents an immediate point of departure from Euclidean spaces: the interpolant must trace a curve on $\gM$. Therefore, the generalised interpolant, $I: [0,1] \times \gM^2 \to \gM$, satisfying the same endpoint constraints, $I_0( x_0, x_1) = x_0$ and $I_1(x_0, x_1) = x_1$, may not any more be any arbitrary linear combination of $x_0$ and $x_1$, as it may not belong to the manifold of interest. Consequently, we parametrise the interpolant as the geodesic connecting the points $x_0$ and $x_1$, $I:(t, x_0, x_1)\mapsto \exp_{x_0}(\alpha_t \log_{x_0}(x_1))$, with $\alpha_0 = 0$ and $\alpha_1 = 1$~\citep{rfm}. This chosen form of $I_t$ also enables us to write the vector field of the flow in~\eqref{eq:prob_flow_on_manifolds} as $\partial_t I_t(x_0, x_1) = \alpha_t'\log_{x_t}(x_1) / (1-\alpha_t)\in \gT_{x_t}\M$. Note that when the manifold is an Euclidean space (\ie $\gM = \R^d$), we recover~\citet{albergo2022building}.

\looseness=-1
Given the stochastic interpolant, we now define \nameshort that jumps along the trajectory of~\eqref{eq:prob_flow_on_manifolds}.

\begin{mdframed}[style=MyFrame2]
\begin{definition}
\label{def:gen_flow_map}
\looseness=-1
(\namelong) Let $(\M,g)$ be a Riemannian manifold, and let $\rho_0$ and $\rho_1$ be two distributions on $(\M,g)$. The generalised flow map is the unique function $X:[0,1]^2\times\M\to\M$, such that, for any solution $(x_t)_{t\in [0,1]}$ of \cref{eq:prob_flow_on_manifolds}, and any $(s,t)\in[0,1]^2$, $X_{s,t}(x_s) = x_t$.
\end{definition}
\end{mdframed}

Analogous to the Euclidean case, \nameshort enables one-step sampling by first sampling $x_0\sim\rho_0(x_0)$ and then applying $X_{0,1}(x_0)\sim\rho_1$. A natural parametrisation for constructing the \nameshort is 
$X_{s,t}(x_s) = \exp_{x_s}\left((t-s)v_{s,t}(x_s)\right)$, with the underlying vector field $v: [0,1]^2 \times \M \to \T_{x_s} \M $.
This automatically satisfies the boundary condition $X_{s, s}(x_s) = x_s$, since $\exp_{x_s}(\vec 0) = x_s$. 
We can thus characterise a \nameshort in three different ways, generalising the Euclidean case. 
%
%

\begin{mdframed}[style=MyFrame2]
\begin{restatable}{proposition}{propgfm}
\label{prop:gfm_characterization}
Let $X_{s,t}$ be parametrised as $X_{s,t}(x) = \exp_{x}\left( (t-s) v_{s,t}(x)\right)$. 
Then $X_{s,t}$ is the unique \nameshort for~\eqref{eq:prob_flow_on_manifolds} if and only if it satisfies any of the following conditions:
\begin{enumerate}[topsep=0pt, partopsep=0pt, itemsep=0pt, parsep=0pt, leftmargin=*, label=(\roman*)]
\item Generalised Lagrangian Condition:
\begin{equation}
    \forall (s,t)\in[0,1]^2,x_s \in \M,\qquad\partial_t X_{s,t}(x_s) = v_{t}(X_{s,t}(x_s)), 
    \label{eq:generalised_lagrangian_condition}
\end{equation}
\item Generalised Eulerian Condition:
\begin{equation}
    \forall (s,t)\in[0,1]^2,x_s\in \M\qquad\partial_s X_{s,t}(x_s) + \ddd (X_{s,t})_{x_s}[v_{s}(x_s)] = 0, 
    \label{eq:generalised_eulerian_condition}
\end{equation}

\item Generalised Semigroup Condition:
\end{enumerate}
\begin{equation}
    \forall(s,t,u) \in [0,1]^3, x_s\in\M,\qquad X_{ u, t}\left(X_{s, u}(x_s)\right) = X_{s,t}(x_s).
    \label{eq:generalised_semi_group_condition}
\end{equation}
\end{restatable}
\end{mdframed}
\looseness=-1
We include the proofs of this proposition in~\S\ref{app:legality_proof} alongside proofs of the legality of the claims. As stated in~\cref{prop:gfm_characterization}, the extension to Riemannian manifolds requires the use of manifold operations such as the differential in place of the Euclidean gradient. Moreover, instantaneously, the \nameshort recovers the vector field in its derivative and defines an implicit flow.

\begin{mdframed}[style=MyFrame2]
\begin{restatable}{lemma}{generalisedtangentcondition}
(Generalised Tangent Condition)
Let $v_t(x_t) = \E_{x_t \sim \rho_t(x_t)}[\partial_t I_t\mid I_t = x_t]$ for any $t$ and $x_t \in \gM$ be the drift of~\eqref{eq:prob_flow_on_manifolds}. Then it holds that $\lim_{s\to t} \partial_t X_{s,t}(x_s) = v_{t}(x_t)$.
\label{lemma:generalisedtangentcondition}
\end{restatable}
\end{mdframed}
\looseness=-1
The proof for~\cref{lemma:generalisedtangentcondition} is provided in~\S\ref{app:lemma_generalised_tangent_condition_proof} and also illustrated in~\cref{fig:intro_fig}. The lemma underscores the key idea that, for $s = t$ (on the diagonal of the times space, $[0,1]^2$), the derivative of the \nameshort is the instantaneous vector field, $v_t$ of~\eqref{eq:prob_flow_on_manifolds}.
Moreover, from the parametrisation in~\cref{prop:gfm_characterization}, and from the fact that $\lim_{s\to t}\partial_tX_{s,t}(x) = v_{t,t}(x)$ for any $x$, it follows that $v_{t,t} = v_t$ for any $t \in [0,1]$. (We thoroughly prove the limit in~\cref{lemma:der_flowmap}.)
Therefore, we hereinafter use $v_{t,t}$ and $v_t$ interchangeably, and we may train the \nameshort on the diagonal $s=t$ using Riemannian Flow Matching (RFM)~\citep{rfm}, the loss of which is given by:
\begin{equation}
    \gL_{\mathrm{RFM}}(\theta) = \E_{(x_0, x_1) \sim \rho(x_0,x_1), x_t \sim \rho_t(x_t)} \left[ \|  v^{\theta}_{t,t}(x_t) - \partial_t I_t(x_0, x_1) \|^2_g \right].
    \label{eq:rfm}
\end{equation}

\subsection{Generalised Self-Distillation losses}
\label{sec:riemannian_self_distillation}
\looseness=-1
Given the above characterisations of the \nameshort in~\cref{prop:gfm_characterization}, we define three new generalised \emph{self-distillation} objectives, leading to Generalised Lagrangian self-distillation (G-LSD), Generalised Eulerian self-distillation (G-ESD), and Generalised progressive self-distillation (G-PSD). 
Specifically, we consider self-distillation objectives of the following form:
\begin{equation}
    \gL(\theta) = \gL_{\mathrm{RFM}}(\theta) + \gL_{\mathrm{GFM\text{-}SD}}\left(\theta\right).
    \label{eq:general_rfm_gfm}
\end{equation}

\looseness=-1
To adapt to the Riemannian setting, we further constrain our vector field $v^{\theta}$ to lie on the tangent plane at the point where it is evaluated. Specifically, letting $f^\theta:[0,1]^2 \times \gM \to \R^d$ be our underlying neural network, then, for any $0 \leq s \leq t \leq 1$, $ p\in \M$, $v^{\theta}_{s,t}(p) \coloneqq \proj_{\gT_p\M}\left(f^\theta_{s,t}(p)\right)$. This ensures that all the usual Riemannian manifold operations required are not ill-defined.

\looseness=-1
\xhdr{Generalised Lagrangian self-distillation}
As linear combination of vectors also belong to the same tangent space, we may freely consider the difference between $X_{s,t}^\theta(I_s)$ and $v_{t,t}^\theta(X_{s,t}^\theta(I_s))$. However, we further adapt the loss to incorporate the Riemannian metric in the norm of the resulting vector so that the losses are comparable with the flow matching ones (namely in terms of magnitude).

\begin{mdframed}[style=MyFrame2]
\begin{restatable}{proposition}{propglsd} 
\label{prop:glsd}
(Generalised Lagrangian self-distillation) The \nameshort is the unique minimiser of the objective in~\eqref{eq:general_rfm_gfm} for $\gL_{\mathrm{GFM\text{-}SD}}\left({\theta}\right) = \gL_{\mathrm{G\text{-}LSD}}(\theta)$, where
\begin{equation}
\label{eq:rlsd}
\gL_{\mathrm{G\text{-}LSD}}({\theta}) = \E_{t,s, (x_0,x_1)}\left[\left\lVert \partial_t X^{\theta}_{ s,t}(I_s) - v^{\theta}_{ t,t}(X^{\theta}_{s,t}(I_s)) \right\rVert_g^2\right].
\end{equation}
\end{restatable}
\end{mdframed}
The proof for~\cref{prop:glsd} is provided in~\S\ref{app:proofs_generalised_losses}.

\looseness=-1
\xhdr{Eulerian self-distillation}
In an analogous manner to the Generalised Lagrangian self-distillation loss, we can instantiate an Eulerian loss to the Riemannian setting.

\begin{mdframed}[style=MyFrame2]
\begin{restatable}{proposition}{propgesd} 
\label{prop:gesd}
(Generalised Eulerian self-distillation) The \nameshort is the unique minimiser of the objective in~\eqref{eq:general_rfm_gfm} for $\gL_{\mathrm{GFM\text{-}SD}}\left({\theta}\right) = \gL_{\mathrm{G\text{-}ESD}}\left({\theta}\right)$, where
\begin{equation}
\label{eq:gesd}
\gL_{\mathrm{G\text{-}ESD}}({\theta}) = \E_{t,s,(x_0,x_1)}\left[\left\lVert \partial_s X_{s,t}^\theta(x_s) + \ddd( X_{s,t}^\theta)_{I_s}[v^\theta_{s,s}(I_s)]\right\rVert^2_g\right].
\end{equation}
\end{restatable}
\end{mdframed}

\looseness=-1
%
For completeness, the proof for~\cref{prop:gesd} is included in~\S\ref{app:proofs_generalised_losses}.

\looseness=-1
\xhdr{Progressive self-distillation}
Naturally, $X_{u,t}\circ X_{s,u}:\M\to\M$ is well-defined for all $(s, u ,t) \in [0,1]^3$. 
Thus, we may use the geodesic distance, $d_g$, to derive the Generalised Progressive Self-Distillation (G-PSD) objective to enforce the semigroup condition~\eqref{eq:generalised_semi_group_condition}.
The G-PSD objective is the simplest to port, as it is devoid of any spatial or time derivatives of the \nameshort as stated in the following proposition. The proof is deferred to~\S\ref{app:proofs_generalised_losses}.
\begin{mdframed}[style=MyFrame2]
\begin{restatable}{proposition}{propgpsd} 
\label{prop:gpsd}
(Generalised Progressive Self-Distillation) The \nameshort is the unique minimiser over $v_{\theta}$ of~\eqref{eq:general_rfm_gfm} for $\gL_{\mathrm{GFM\text{-}SD}}\left({\theta}\right) = \gL_{\mathrm{G\text{-}PSD}}\left({\theta}\right)$ and $u \mid (s,t) \sim \gU(s, t)$, where
\begin{equation}
\label{eq:gpsd}
\gL_{\mathrm{G\text{-}PSD}}({\theta}) = \E_{t,s,u,(x_0,x_1)}\left[d_g^2\left(X^{\theta}_{s,t}(I_s), X^{\theta}_{u,t}\left(X^{\theta}_{s,u}(I_s)\right)\right)\right].
\end{equation}
\end{restatable}
\end{mdframed}

\subsection{Training Generalised Flow Maps}
\label{sec:training_gfm_with_stopgrad}

\looseness=-1
While it is possible to implement the \nameshort objectives naively, it may incur unstable training dynamics. Instead, and in line with the literature~\citep{ayf,flowmaps,meanflows}, we opt for a \emph{self-bootstrapped} objective by placing a stop-gradient operator to only optimise parts of each objective. This, in turn, converts one term in the objective as the ``teacher'', which is distilled to terms in the loss that have non-zero parameter gradients. 
We include in~\cref{al:gfm_training} the pseudocode used to train a \nameshort for any self-distillation loss.



\cut{
 However, each of these losses, if trained on rawly, involves intricate combinations of derivatives and Jacobians which can natively be handled by modern ML libraries, but may lead to unstable optimisation objectives---as noted out in~\citet{flowmaps}, for example. Therefore, it is common practice~\citep{ayf,flowmaps,meanflows} to place a stop-gradient operator to only optimise on some of the members of the objectives, instead of all at once. Different \emph{valid} combinations lead to different optimisation objectives, and hence incur different training dynamics. Note, here, that putting a stop-gradient operator on one of the terms makes it a ``teacher'', as it is therefore the reference for the part of the loss that has non-zero gradients with respect to the parameters.
We include in~\cref{al:gfm_training} the pseudocode used to train a \nameshort for any choice of self-distillation loss.
}

\looseness=-1
\xhdr{G-LSD loss} We place the stop-gradient on the second term of~\eqref{eq:generalised_lagrangian_condition}, yielding the objective:
\begin{equation}
\hat \gL_{\mathrm{G\text{-}LSD}}({\theta}) = \E_{t,s, (x_0,x_1)}\left[\left\lVert \partial_t X^{\theta}_{ s,t}(I_s) - \sg\left(v^{\theta}_{ t,t}(X^{\theta}_{s,t}(I_s)) \right)\right\rVert_g^2\right].
\end{equation}
Back-propagating through $\partial_t X_{s,t}^\theta$ is simple in usual modern ML libraries, as forward-mode automatic differentiation is typically available alongside the Jacobian-Vector Product (JVP) operation.

\looseness=-1
\xhdr{G-ESD loss} Similarly, for the Eulerian loss, we apply the stop-gradient to the second term, which contains the spatial derivative---and thus avoiding higher order derivatives---leading to:
\begin{equation}
    \hat \gL_{\mathrm{G\text{-}ESD}}(\theta) = \E_{t,s, (x_0,x_1)}\left[\left\lVert \partial_s X_{s,t}(I_s) + \sg\left(\ddd(X_{s,t}^\theta)_{I_s}[v_{s,s}^\theta(I_s)]\right)\right\rVert_g^2\right].
\end{equation}
Note that this results in an objective closely related to that of (a Riemannian generalisation of) Mean Flows~\citep{meanflows}---a connection we prove in~\S\ref{app:connection_existing_methods}. We also implement the latter, and term this objective G-MF (Generalised Mean Flows), and define it fully in~\S\ref{app:connection_existing_methods}. This objective is trained as in Mean Flows; that is to say, without the flow matching loss.

\looseness=-1
\xhdr{G-PSD loss} Finally, we also utilise the stop-gradient operator on the PSD objective using the two, smaller steps ($s$ to $u$, $u$ to $t$) as the teacher for the larger step ($s$ to $t$), resulting in the following:
\begin{equation}
\hat \gL_{\mathrm{G\text{-}PSD}}({\theta}) = \E_{t,s,u,(x_0,x_1)}\left[d_g^2\left(X^{\theta}_{s,t}(I_s), \sg\left(X^{\theta}_{u,t}\left(X^{\theta}_{s,u}(I_s)\right)\right)\right)\right].
\end{equation}
\looseness=-1
We follow~\citet{flowmaps} by setting $u = \frac 12 s +\frac 12 t$, leading to two half-steps, thus generalising shortcut models of~\citet{shortcut} to the manifolds---a connection we also detail in~\S\ref{app:connection_existing_methods}.

\begin{algorithm}[tb]
\caption{\nameshort training, for any choice of self-distillation.}
\KwIn{Riemannian manifold, $(\M, g)$; time distributions, $T$, $S\mid T$; coupling, $\rho$; batch size, $M$.}

\Repeat{converged}{
  Draw batch $(t_i, s_i,x_0^i,x_1^i)_{i=1}^M \sim (T, S\mid T, \rho(x_0,x_1)$)\;
  Construct $x_s^i = I(x_0^i, x_1^i, s^i)$, compute $u_s^i = \partial_{s^i} I(x_0^i, x_1^i, s^i)$\;
  Estimate $\hat \gL_{\mathrm{RFM}}(\theta,(x_s^i)_{i=1}^M, (u_s^i)_{i=1}^M) \approx \gL_\mathrm{RFM}(\theta)$\;
  Estimate $\hat \gL_{\mathrm{GFM\text{-}SD}}(\theta, x_s^i, s^i, t^i) \approx \gL_{\mathrm{GFM\text{-}SD}}(\theta)$\;
  Optimisation step on $\hat \gL_{\mathrm{RFM}}(\theta,(x_s^i)_{i=1}^M, (u_s^i)_{i=1}^M) + \hat \gL_{\mathrm{GFM\text{-}SD}}(\theta, (x_s^i)_{i=1}^M, (s^i)_{i=1}^M, (t^i)_{i=1}^M)$\;
}
\KwOut{Flow map $X^\theta$.}
\label{al:gfm_training}
\end{algorithm}

%% file: section/experiments.tex
\section{Experiments}
\label{sec:experiments}

\looseness=-1
We test the empirical calibre of \nameshort on a suite of standard geometric generative modelling benchmarks. \ifarxiv\footnote{We make our code publicly available at the following URL: \url{https://github.com/olsdavis/gfm}.}\else\fi Specifically, we instantiate \nameshort on torsion angles ($\mathbb{T}^2 \cong \mathbb{S}^1 \times \mathbb{S}^1$) found in protein side chains and RNA backbones $\mathbb{T}^7 \cong (\mathbb{S}^1)^7$~\citep{lovell2003,murray2003}, catastrophic geospatial events on Earth ($\mathbb{S}^2$) as introduced in~\citet{mathieu2020riemannian}, a synthetic dataset of the manifold of $3\text{D}$ rotations ($\sothree$), and on the Poincar\'e disk (hyperbolic geometry).

\looseness=-1
\xhdr{Metrics}
To evaluate the few-step sample quality of our models, we use the empirical MMD (Maximum Mean Discrepancy) between the test-set and the samples, with an RBF kernel using the manifold's distance, $d_g$, and a bandwidth of $\kappa = 1$. See~\S\ref{app:additional_experimental_details} for the calculation details. To assess the learnt vector field, we also compute the negative log-likelihood (NLL) on the test set when available. {We have included in~\cref{app:nll_relevance} a discussion explaining why NLL is relevant in this context.} Finally, we provide qualitative samples to further assert the method's soundness.

\looseness=-1
\xhdr{Baselines} We use RFM~\citep{rfm}, which is the state-of-the-art method, as the main baseline. Additionally, we include results for a Riemannian diffusion model (RDM) ~\citep{huang2022riemannian} and Riemannian score-based generative models (RSGM)~\citep{de2022riemannian} with test NLL results taken directly from the respective papers due to a lack of open source code. We also report a mixture of power spherical distributions (MoPS)~\citep{de2020power} for test NLL.

{Finally, we also detailed the computational cost of having using our proposed methods against RFM in~\cref{app:compute}.}

\subsection{Proteins Torsion Angles and RNA backbones on flat tori}
\label{sec:experiments_proteins_main_big_table}

\begin{table}[!t]
    \centering
\caption{Test NLL on protein sidechain and RNA torsion angles. Standard deviation estimated over $5$ runs. $^\dagger$ indicates baseline numbers taken from~\citet{huang2022riemannian}.
}
\vspace{-5pt}
\label{tab:nll_protein}
\resizebox{1\linewidth}{!}{
\begin{tabular}{lccccc}
\toprule
 & \textbf{General (2D)} & \textbf{Glycine (2D)} & \textbf{Proline (2D)} & \textbf{Pre-Pro (2D)} & \textbf{RNA (7D)} \\
\midrule
\textbf{Dataset size} & $138{,}208$ & $13{,}283$ & $7{,}634$ & $6{,}910$ & $9{,}478$ \\
\midrule
 MoPS$^\dagger$ & $1.15 {\pm 0.002}$ & $2.08 {\pm 0.009}$ & $0.27 {\pm 0.008}$ & $1.34 {\pm 0.019}$ &  $4.08 \pm 0.368$ \\
RDM$^\dagger$~\citep{huang2022riemannian} & $1.04 {\pm 0.012}$ & $ 1.97 {\pm 0.012}$ & $ 0.12 {\pm 0.011}$ & $ 1.24 \pm0.004$  & $ -3.70 {\pm0.592}$\\
RFM~\citep{rfm}& ${1.01}\pm0.025$ & $\mathbf{1.90 \pm 0.055}$ & $0.15 \pm 0.027$ & $ 1.18  \pm 0.055$ & $\mathbf{-5.20 \pm 0.067}$ \\
\cmidrule(lr){1-6}
\ourslsd (ours) & $0.99\pm 0.05$ & $1.99\pm0.02$ & $0.24\pm0.07$ & $1.11\pm0.02$ & $-4.15 \pm 0.09$ \\
\ourspsd (ours) & $\mathbf{0.95\pm 0.02}$ & ${1.94\pm0.03}$ & $\mathbf{0.08\pm 0.04}$ & $1.10\pm0.04$ & $-4.40 \pm 0.13$ \\
\oursesd (ours) & $0.99\pm0.04$ & $1.95\pm0.01$ & $0.19\pm 0.04$ & $1.10\pm0.02$ & $-4.61 \pm 0.07$ \\
\oursmf (ours) & $0.97\pm 0.01$ & $1.97\pm0.01$ & $0.21\pm0.04$ & $\mathbf{1.02\pm0.04}$ & $-3.79 \pm 0.09$ \\
\bottomrule
\end{tabular}
}
\end{table}

\looseness=-1
We train our model on a protein dataset on the flat 2D and 7D (RNA) tori ($\mathbb{S}^1\times\mathbb{S}^1$ and $(\mathbb{S}^1)^7$). The 2D data is from~\citet{lovell2003} and the 7D data from~\citet{murray2003}, compiled together by~\citet{riemannianDiffusion}. We report our MMD results in~\cref{tab:nll_protein,tab:protein_mmd_1nfe} and qualitative results in ~\cref{fig:nfe_mmd_proteins,fig:comparison-rama}. 
%
We observe that the test NLLs produced by \nameshort outperform those of RFM for the protein side chain datasets, and are marginally worse on the RNA dataset. Most importantly, we test the \emph{one-step} generative capability that is unique to \nameshort in~\cref{tab:protein_mmd_1nfe,fig:nfe_mmd_proteins} and find that \nameshort offers considerable gains, especially \ourslsd, which offers an improvement of up to $22\times$ on the MMD (on ``General'') for a single function evaluation. We also plot Ramachandran plots on the torsion angles in~\cref{fig:comparison-rama}, where we observe that our methods achieve log-likelihood landscapes comparable to those of RFM. Finally, as an ablation, we also plot in~\cref{fig:nfe_mmd_proteins} the MMD as a function of inference steps and find that \nameshort consistently improve on the MMD for low NFEs, with quality consistently improving with higher steps as RFM. 

\begin{figure}[!htb]
    \caption{Ramachandran plots on the General protein dataset. Test-set samples depicted in red.
    }
    \vspace{-10pt}
    \centering
    \begin{subfigure}{0.19\linewidth}
        \centering
        \includegraphics[width=\linewidth]{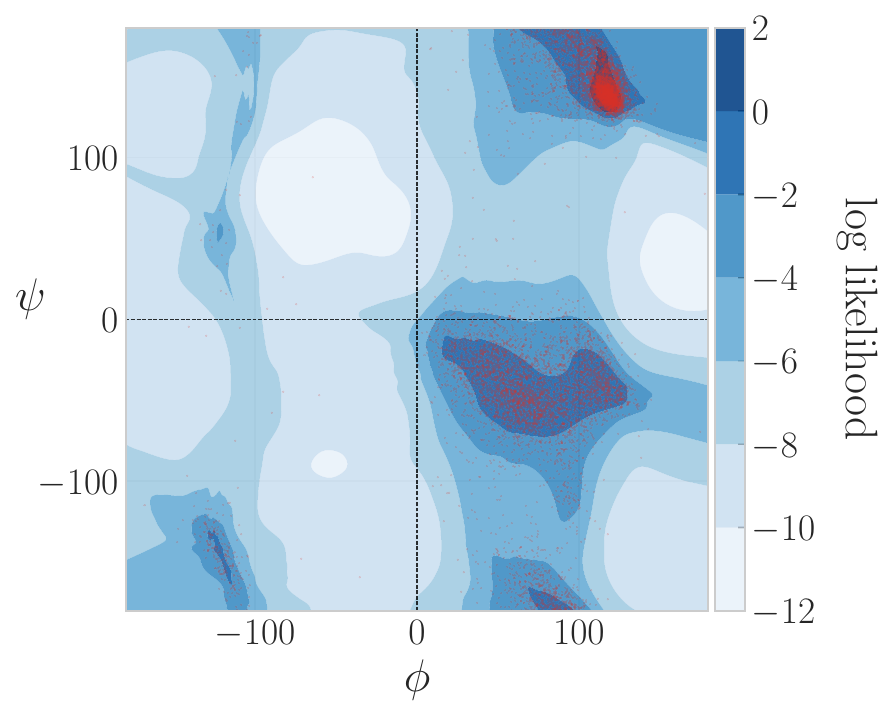}
        \caption{G-ESD}
        \label{fig:rama-esd}
    \end{subfigure}
    \begin{subfigure}{0.19\linewidth}
        \centering
        \includegraphics[width=\linewidth]{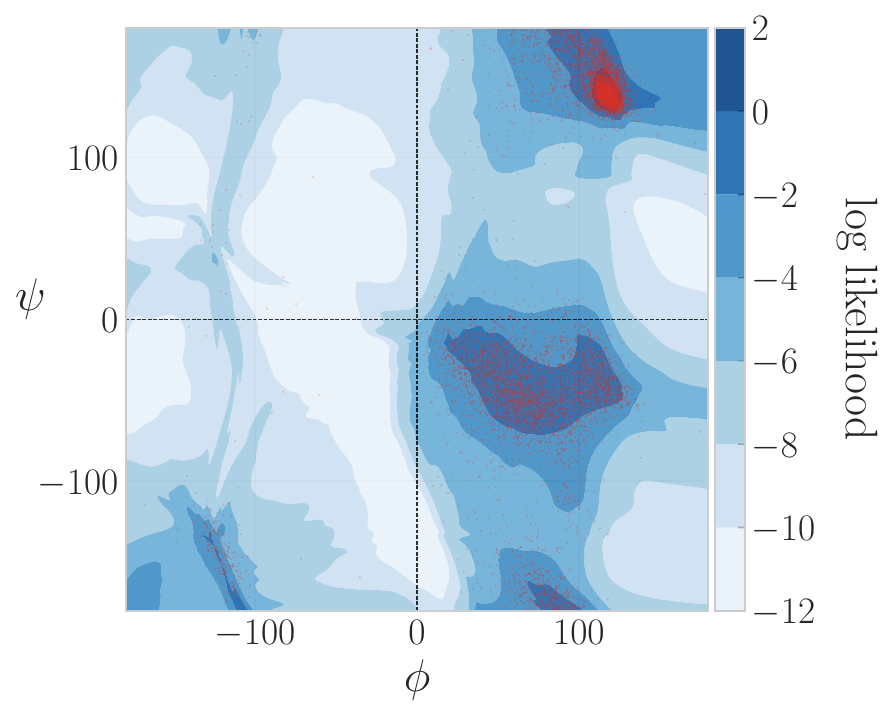}
        \caption{G-LSD}
        \label{fig:rama-lsd}
    \end{subfigure}
    \begin{subfigure}{0.19\linewidth}
        \centering
        \includegraphics[width=\linewidth]{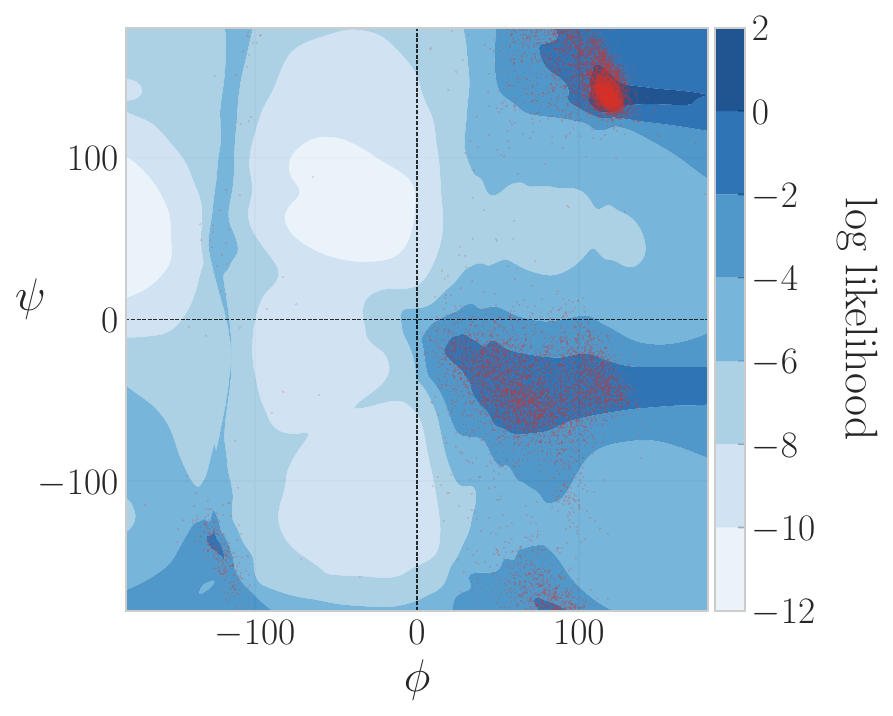}
        \caption{G-PSD}
        \label{fig:rama-psd}
    \end{subfigure}
    \begin{subfigure}{0.19\linewidth}
        \centering
        \includegraphics[width=\linewidth]{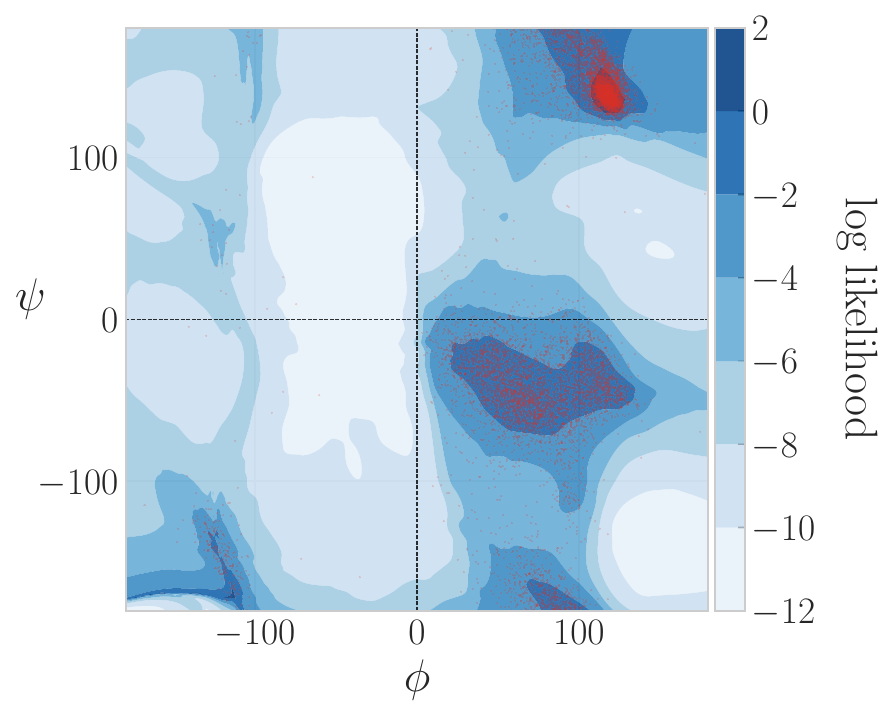}
        \caption{G-MF}
        \label{fig:rama-mf}
    \end{subfigure}
    \begin{subfigure}{0.19\linewidth}
        \centering
        \includegraphics[width=\linewidth]{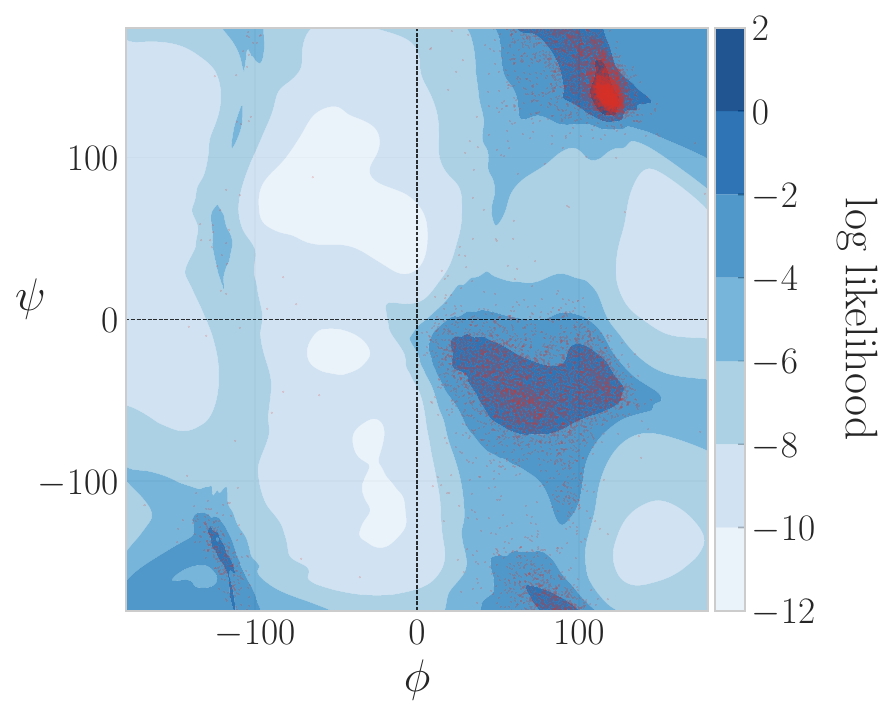}
        \caption{RFM}
        \label{fig:rama-rfm}
    \end{subfigure}
    \label{fig:comparison-rama}
    \vspace{-10pt}
\end{figure}

\begin{figure}[!t]
    \centering
    \caption{MMD on protein datasets against the NFE.}
    \vspace{-10pt}
    \includegraphics[width=\linewidth]{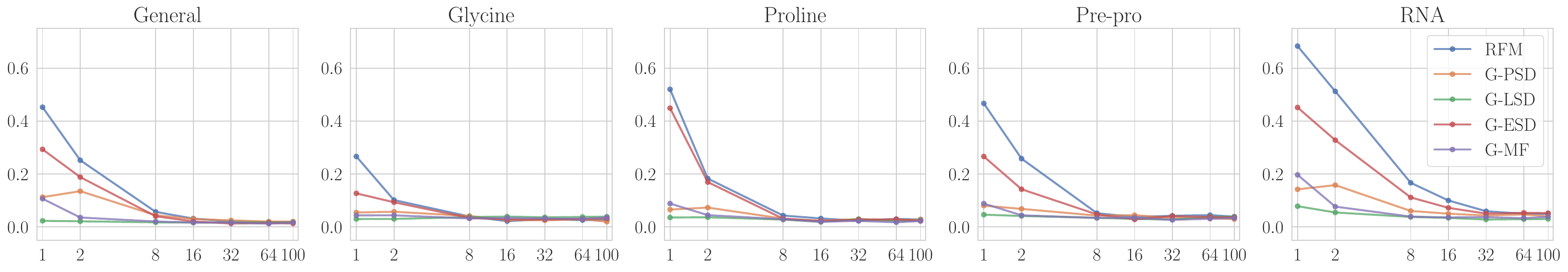}
    \label{fig:nfe_mmd_proteins}
    \vspace{-10pt}
\end{figure}

\begin{table}[tbp]
\centering
\caption{MMD for 1 NFE with the test-set for proteins torsion angles and RNA backbones. Standard deviation estimated over 5 seeds.}
\vspace{-5pt}
\label{tab:protein_mmd_1nfe}
\resizebox{\linewidth}{!}{
\begin{tabular}{lccccc}
\toprule
 & \textbf{General (2D)} & \textbf{Glycine (2D)} & \textbf{Proline (2D)} & \textbf{Pre-Pro (2D)} & \textbf{RNA (7D)} \\
\midrule
RFM~\citep{rfm} & $0.45 \pm 0.006$ & $0.27 \pm 0.008$ & $0.52 \pm 0.057$ & $0.47 \pm 0.022$ & $0.68\pm 0.011$ \\
\cmidrule(lr){1-6}
\ourslsd (ours) & $\mathbf{0.02 \pm 0.003}$ & $\mathbf{0.03 \pm 0.004}$ & $\mathbf{0.04 \pm 0.012}$ & $\mathbf{0.05 \pm 0.004}$ & $\mathbf{0.08 \pm 0.007}$ \\
\ourspsd (ours) & $0.11 \pm 0.016$ & $0.05 \pm 0.019$ & $0.07\pm 0.011$ & $0.08 \pm 0.015$ & $0.14 \pm 0.027$ \\
\oursesd (ours) & $0.29\pm0.002$ & $0.13 \pm 0.006$ & $0.44 \pm 0.024$ & $0.26 \pm 0.016$ & $0.45 \pm 0.006$ \\
\oursmf (ours) & $0.11\pm0.029$ & $0.04 \pm 0.019$ & $0.09 \pm 0.046$ & $0.09 \pm 0.017$ & $0.20 \pm 0.019$ \\
\bottomrule
\end{tabular}
}
\end{table}

\subsection{Earth catastrophes on the 2-sphere}
\label{sec:experiments_on_earth}

\looseness=-1
We evaluate \nameshort on a collection of geospatial data, represented on the 2-sphere, $\mathbb{S}^2$. The dataset was first introduced in~\citet{riemCNF} and is curated from various sources~\citep{NOAA2020a,NOAA2020b,Brakenridge2017,EOSDIS2020}. We report our results in~\cref{fig:nfe_mmd_earth,tab:nll_earth}. Analogous to tori, we observe that the implicit flow $v^{\theta}_{t,t}$ within \nameshort offers log-likelihoods that outperform those of RFM and all other methods on three out of four datasets, with the exception of ``Volcanoes''. 
To assess sample quality, we again plot the MMD as a function of the number of integration steps, and observe once more great improvement on MMD for all our methods at 1 or 2 NFEs, while preserving high quality for higher NFEs. Finally, we provide the densities learnt by RFM and G-LSD on 2D Earth plots in~\cref{fig:earth_all_densities} (with the remaining plots in~\cref{fig:other_earth}), and remark similar likelihood landscapes between RFM and G-LSD. 

\begin{table}[h]
\centering

\caption{Test NLL on the Earth datasets. Standard deviation is estimated over 5 runs. $^\dagger$ indicates baseline numbers taken from~\citet{huang2022riemannian}.}
\vspace{-5pt}
\label{tab:nll_earth}
\resizebox{1\linewidth}{!}{
\begin{tabular}{lcccc}
\toprule
 & \textbf{Volcano} & \textbf{Earthquake} & \textbf{Flood} & \textbf{Fire} \\
\midrule
\textbf{Dataset size} & $827$ & $6{,}120$ & $4{,}875$ & $12{,}809$ \\
\midrule
Mixture of Kent$^\dagger$ & $-0.80 {\pm 0.47}$ & $0.33 {\pm 0.05}$ & $0.73 {\pm 0.07}$ & $-1.18 {\pm 0.06}$ \\
Riemannian CNF$^\dagger$ ~\citep{mathieu2020riemannian} & $-0.97 {\pm 0.15}$ & $0.19 {\pm0.04}$ & $0.90 {\pm0.03}$ & $-0.66 {\pm0.05}$ \\
    Moser Flow$^\dagger$  \citep{rozen2021moser} & $-2.02 {\pm 0.42}$ & $-0.09 {\pm0.02}$ & $0.62 {\pm 0.04}$ & $-1.03 {\pm 0.03}$ \\
Stereographic Score-Based$^\dagger$  & ${-4.18} {\pm 0.30}$ & ${-0.04} {\pm 0.11}$ & ${1.31} {\pm 0.16}$ & $0.28 {\pm 0.20}$ \\
RSGM$^\dagger$  \citep{de2022riemannian} & $-5.56 {\pm0.26}$ & $-0.21 {\pm0.03}$ & $0.52 {\pm0.02}$ & $-1.24 {\pm 0.07}$\\
RDM$^\dagger$ ~\citep{huang2022riemannian} & $-6.61 {\pm 0.97}$ & $-0.40 {\pm 0.05}$ & $0.43 {\pm 0.07}$ & $-1.38 {\pm0.05}$\\
RFM~\citep{rfm} & $\mathbf{-7.93 \pm 1.67}$ & $-0.28 \pm 0.08$ & $0.42 \pm 0.05$ & $-1.86 \pm 0.11$ \\
\cmidrule(lr){1-5}
\ourslsd (ours) & $-4.96\pm0.68$ & $-0.93\pm 0.01$ & $-0.38\pm0.33$ & $-2.14\pm0.42$\\
\ourspsd (ours) & $-3.50\pm0.22$ & $-0.63\pm0.13$ & ${-0.76\pm0.13}$ & $\mathbf{-2.48\pm0.71}$\\
\oursesd (ours) & $-4.49\pm0.20$ & $-0.67\pm0.08$ & $\mathbf{-0.88\pm0.38}$ & $-2.29\pm0.08$\\
\oursmf (ours) & $-3.73\pm0.41$ & $\mathbf{-1.08\pm0.09}$ & $-0.72\pm0.11$ & $-2.24\pm0.30$\\
\bottomrule
\end{tabular}
}
\vspace{-10pt}
\end{table}

\begin{figure}[!hbp]
    \centering
    \caption{MMD on Earth datasets against the NFE.}
    \vspace{-5pt}
    \includegraphics[width=\linewidth]{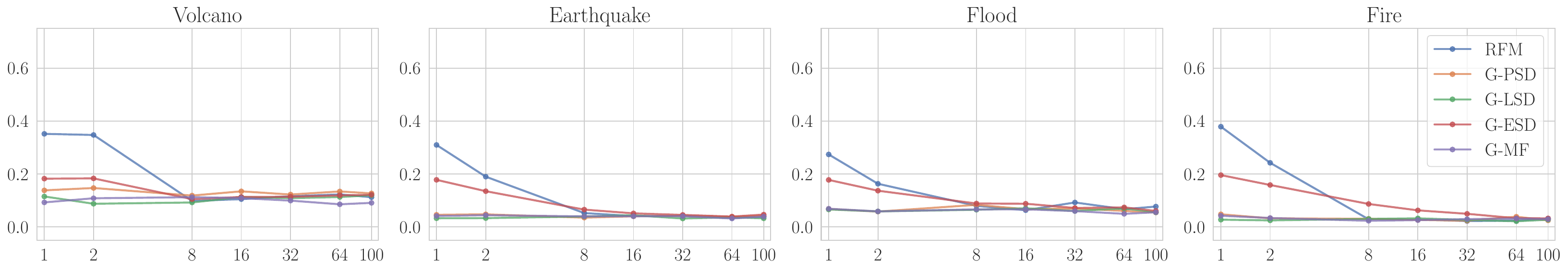}
    \label{fig:nfe_mmd_earth}
\end{figure}

\begin{figure}[htp]
  \centering
  \caption{Plots of densities for the various datasets and all compared methods. Depicted in red are the test-set samples. Datasets from left to right: volcano, earthquake, flood, fire.}
  \vspace{-10pt}
  \label{fig:earth_all_densities}
  \resizebox{\linewidth}{!}{
  \begin{tabular}{m{0.1\linewidth}m{0.9\linewidth}}
    RFM
      & \includegraphics[width=\linewidth]{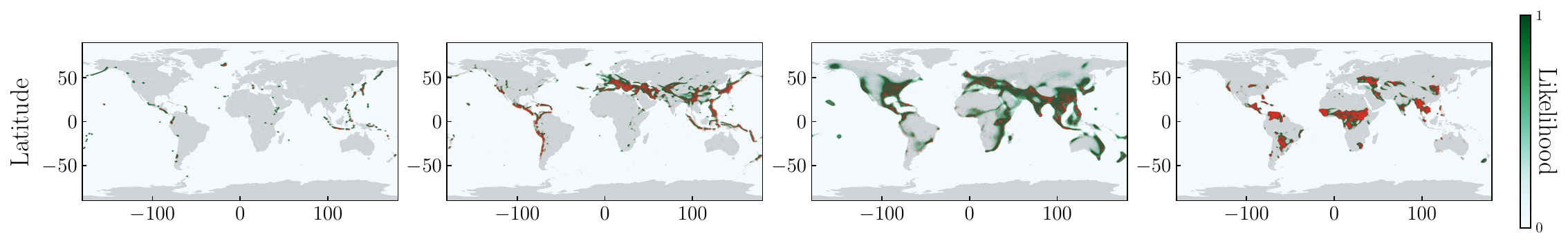}\\
    \ourslsd
      & \includegraphics[width=\linewidth]{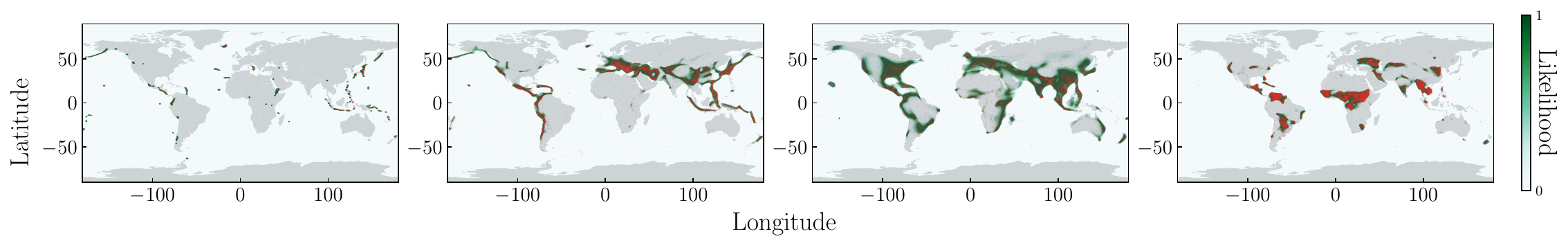} \\
  \end{tabular}
  }
  
\end{figure}

\begin{figure}[tbp]
\begin{minipage}[t]{0.49\linewidth}
    \centering
    \caption{MMD on the hyperbolic dataset against the NFE.}
    \label{fig:nfe_mmd_hyperbolic}
    \includegraphics[width=\linewidth]{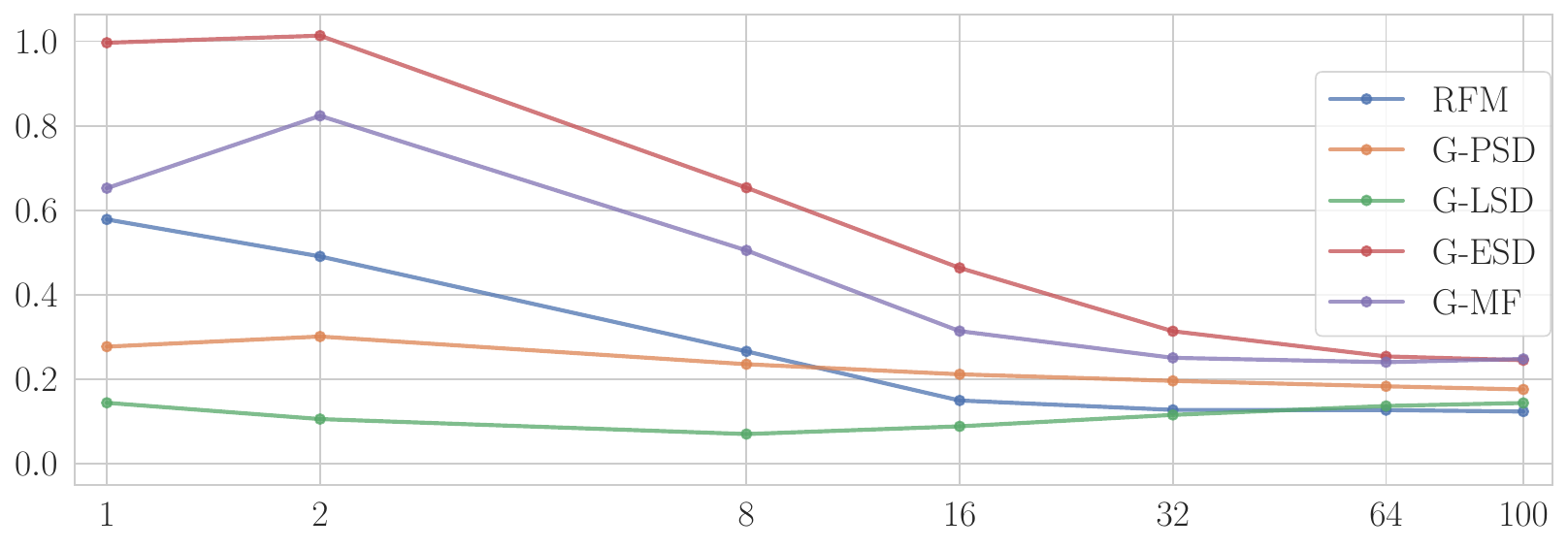}
\end{minipage}
\hfill
\begin{minipage}[t]{0.49\linewidth}
\centering
\captionof{table}{Results on the \sorth~test-set with standard deviation estimated over 5 seeds.}
\label{tab:so3_results}
\resizebox{\linewidth}{!}{
    \begin{tabular}{lccc c}
\toprule
 & \multicolumn{3}{c}{\textbf{MMD}} & \multirow{2}{*}{\textbf{NLL}}\\
\cmidrule(lr){2-4}
 & 1 NFE & 2 NFE & 100 NFE & \\
\midrule
RFM & $0.147\pm0.007$ & $0.083\pm0.003$ & $0.042\pm0.002$ & ${-7.15\pm0.03}$ \\
\midrule
\ourslsd (ours) & $\mathbf{0.064\pm0.007}$ & $\mathbf{0.059\pm0.005}$ & $0.044\pm 0.008$ & $-7.11\pm0.03$ \\
\ourspsd (ours) & $0.121\pm0.01$ & $0.073\pm0.005$ & $\mathbf{0.039\pm0.004}$ & ${-7.15\pm0.03}$ \\
\oursesd (ours) & $0.411\pm0.001$ & $0.408\pm0.001$ & $0.109\pm0.011$ & $\mathbf{-7.20\pm0.02}$ \\
\oursmf (ours) & $0.291\pm0.007$ & $0.280\pm0.015$ & $0.283\pm0.014$ & $-6.85\pm0.03$ \\
\bottomrule
\end{tabular}
}
\end{minipage}
\end{figure}

\subsection{\sorth~synthetic data}
\looseness=-1
We next instantiate \nameshort on the manifold of $3\text{D}$ rotations, \sorth, using the synthetic dataset from~\citet{riemannianDiffusion}. We compute both test NLL and MMD values and report them in~\cref{tab:so3_results}. We find that, on test NLL, all methods perform roughly equally, and there is no clear winner. On MMD, we find that all versions of \nameshort outperform RFM, with $\ourslsd$ being again the most performant. This demonstrates the effectiveness of all our methods on non-trivial manifolds.

\subsection{Hyperbolic Manifolds}
\label{sec:experiments_hyperbolic}

We evaluate \nameshort on a manifold with a non-trivial metric, in particular, on the Poincar\'e ball for hyperbolic geometry. We draw $20{,}000$ samples from a target distribution that is wrapped normal, which we then use to train all methods, including RFM. We report the MMD as a function of the NFEs in~\cref{fig:nfe_mmd_hyperbolic}, and observe that all versions of \nameshort outperform RFM, even at high NFE, except for G-MF and G-ESD. Indeed, it seems that the variance during training was higher than for the other methods, which caused it to under-fit the distribution at lower time steps.


%% file: section/related_work.tex
\section{Related work}
\label{sec:related_work}

\looseness=-1
\xhdr{Riemannian generative models} The most related early efforts to build manifold structure into generative models come from conventional normalising flows built out of iterative coupling layers~\citep{tabak2010, rezende2016variational, dinh2017density}, where each coupling layer was designed to preserve the manifold structure~\citep{bose2020latentvariable, rezende2020normalizing, kanwar2020, boyda2021, bose2021equivariant}. These ideas have also been extended to continuous-time flows and diffusions on general Riemannian structures~\citep{lou2020neuralmanifold, mathieu2020riemannian, falorsi2020neural, debortoli2022riemannian}, and their optimisation has also been made simulation-free~\citep{rozen2021moser, rfm, bose2020latentvariable,fisher}. Here, our focus is on extending these approaches so that the more efficient any-step flow map is well-posed and learnable on general geometries. 

\looseness=-1
\xhdr{Accelerated inference in generative models} Early work on accelerated inference focused on a teacher-student procedure~\citep{song2020denoising,luhman2021knowledge,salimans2022progressive, meng2023distillation} where an expensive inference model is distilled to produce the same output in fewer steps. Orthogonal to this have been efforts to parallel diffusion inference with adaptive~\citep{chen2024accelerating, dockhorn2022geniehigherorder, shih2023parallel, tang2024parallel} or speculative \citep{debortoli2025accelerated} schemes. Consistency models aim to directly learn the one-step map to the data distribution from any point along the trajectory~\citep{song2023consistency, kim2024consistency, song2023improved}. The flow map~\citep{boffi2024flow} has emerged as a unifying picture, and recent efforts have shown how to distill it~\citep{sabour2024align} or directly learn it~\citep{flowmaps, meanflows} using the equations that characterise it. We take the leap here to show how to generalise this complete class of models to the Riemannian setting for performant ends.

\looseness=-1
\xhdr{Concurrent work} Concurrent and most related to our work is that of~\citet{cheng2025riemannianconsistencymodel}. The authors propose a new method, Riemannian Consistency Models (RCM), to train few-step generative models on Riemannian manifolds from scratch. Their approach directly ports the original Consistency Models~\citep{song2023consistency} to Riemannian geometries, relying on more sophisticated geometric constructions. In contrast, \nameshort enjoys simpler practical instantiations as it relies on the self-distillation of the Flow Maps framework~\citep{flowmaps} and recovers shortcut models as a special case of the PSD objective executed on Riemannian manifolds.


%% file: section/conclusion.tex
\section{Conclusion}
\label{sec:conclusion}

\looseness=-1
We propose \namelong, a new class of geometric generative models that are capable of performing few-step inference on arbitrary Riemannian manifolds. To build \nameshort, we provide three equivalent theoretical conditions that characterise a flow map on manifolds, and three corresponding self-distillation objectives. We demonstrate the empirical performance of \nameshort in the low NFE regime and achieve state-of-the-art results in sample-based metrics, and competitive test likelihoods in comparison to RFM. While each \nameshort condition leads to a different corresponding objective, at present, the Lagrangian objective remains the most performant in generating high-quality samples, and understanding this empirical observation from a theoretical lens is a natural direction for future work. Additionally, as we empirically demonstrate, the implicit flow within the flow map may lead to better NLL than the flow learned through RFM, which points to an interesting direction for gaining theoretical understanding in future work.

\section*{Acknowledgements}
OD is funded by both Project CETI and Intel.  AJB is partially supported by an NSERC Post-doc fellowship. This research is partially supported by the EPSRC Turing AI World-Leading Research Fellowship No. EP/X040062/1 and EPSRC AI Hub No. EP/Y028872/1. MSA is supported by a Junior Fellowship at the Harvard Society of Fellows as well as the National Science Foundation under Cooperative Agreement PHY-2019786 (The NSF AI Institute for Artificial Intelligence and Fundamental Interactions, http://iaifi.org/). This work has been made possible in part by a gift from the Chan Zuckerberg Initiative Foundation to establish the Kempner Institute for the Study of Natural and Artificial Intelligence.

%% file: section/statements.tex
\ifarxiv
\else
\section*{Ethics Statement}
We hereby acknowledge and declare to have abided by the ICLR Code of Ethics. Our current work does not utilise any sensitive data, nor does it directly enable nefarious usage, although other derivative works could use ours in unethical contexts. The authors declare no competing financial interests.

\section*{Reproducibility Statement}
All of our results are reproducible and have been run with the same 5 random seeds that were set in advance, and that were never changed throughout our evaluation, to provide us our standard deviations. The datasets we use are entirely public and are freely available to all. Our code will be made public upon acceptance alongside detailed instructions for running it locally, and the exact configuration files that were used.
\fi

%% file: section/appendix.tex
\appendix

\section{Self-Distillation on Euclidean spaces}
\label{app:flow_map_losses_original}
In this section, we provide a self-contained review of the three self-distillation based objective functions introduced by~\citet{flowmaps}.
These extend upon the original distillation-based objectives introduced in~\citet{boffi2024flow}.
To this end, we first recall three key properties of the flow map.
\begin{proposition}[Flow map characterization]
\label{prop:flow_map_char}
Let $\partial_t{x}_t = v_t(x_t)$ denote a probability flow ODE, and let $X_{s, t}: [0, 1]^2 \to \R^d$ denote its flow map.
Then,
\begin{enumerate}[label=\roman*.]
    \item $X_{s, t}$ is the unique solution to the Lagrangian equation,
    \begin{equation}
        \label{eqn:lagrangian}
        \partial_t X_{s, t}(x) = v_t(X_{s, t}(x)) \qquad \forall\: (x, s, t) \in \R^d \times [0, 1]^2.
    \end{equation}
    \item $X_{s, t}$ is the unique solution to the Eulerian equation,
    \begin{equation}
    \label{eqn:eulerian}
        \partial_s X_{s, t}(x) + \nabla X_{s, t}(x)v_s(x) = 0 \qquad \forall\: (x, s, t) \in \R^d \times [0, 1]^2.
    \end{equation}
    \item $X_{s, t}$ satisfies the semigroup condition,
    \begin{equation}
        X_{s, t}(x) = X_{u, t}(X_{s, u}(x)) \qquad \forall\: (x, s, u, t) \in \R^d \times [0, 1]^3.
    \end{equation}
\end{enumerate}
\end{proposition}
\begin{proof}
    We simply prove the forward direction: that the flow map satisfies each property.
    The proof of uniqueness follows similarly and can be found in~\citet{flowmaps}.

    Each property follows from the defining condition
    \begin{equation}
        \label{eqn:jump}
       X_{s, t}(x_s) = x_t \qquad \forall\: (s, t) \in [0, 1]^2. 
    \end{equation}
    Taking the derivative with respect to time, we find that
    \begin{equation}
    \begin{aligned}
        \partial_t X_{s, t}(x_t) &= \partial_t x_t,\\
        &= v_t(x_t),\\
        &= v_t(X_{s, t}(x_s)),
    \end{aligned}
    \end{equation}
    which is the Lagrangian equation.
    Noting that $x_s$ was arbitrary completes the proof.

    Taking the total derivative of~\eqref{eqn:jump} with respect to $s$, we find
    \begin{equation}
       \label{eqn:eulerian_derivation} 
       \begin{aligned}
        \partial_s X_{s, t}(x_s) + \nabla X_{s, t}(x_s)\partial_t{x}_s &= 0,\\
       \implies \partial_s X_{s, t}(x_s) + \nabla X_{s, t}(x_s)b_s(x_s) &= 0. 
       \end{aligned}
    \end{equation}
    Again, noting that $x_s$ was arbitrary completes the proof.

    The semigroup condition follows simply by noting that
    \begin{equation}
        X_{u, t}(X_{s, u}(x_s)) = X_{u, t}(x_u) = x_t = X_{s, t}(x_s).
    \end{equation}
    This completes the proof.
\end{proof}
As discussed in the main text, from the Lagrangian condition in~\eqref{eqn:lagrangian}, we may observe that
\begin{equation}
    \lim_{s \to t}\partial_t X_{s, t}(x) = \lim_{s \to t}v_t(X_{s, t}(x)) = v_t(x)
\end{equation}
assuming continuity of $X$ and using that $X_{t, t}(x) = x$ for all $(x, t) \in \R^d \times [0, 1]$.
Moreover, using the parameterization
\begin{equation}
    \label{eqn:param}
    X_{s, t}(x) = x + (t-s)v_{s, t}(x)
\end{equation}
we find that
\begin{equation}
    \lim_{s \to t}\partial_t X_{s, t}(x) = \lim_{s\to t}\left\{(t-s)\partial_tv_{s, t}(x) + v_{s, t}(x) \right\} = v_{t, t}(x) = v_t(x).
\end{equation}
This observation immediately leads to three ``self-distillation" schemes, each of which proceeds by training the flow $v_{t, t}(x)$ on the diagonal $s = t$ via flow matching, and simultaneously distilling it into a flow map by minimizing the square residual on one of the conditions in~\cref{prop:flow_map_char}.

\begin{proposition}[Euclidean self-distillation~\citep{flowmaps}]
\label{prop:sd_euclid}
Let $X^{\theta}_{s, t}(x) = x + (t-s)v^{\theta}_{s, t}(x)$ denote a candidate flow map, and assume that $v^{\theta}$ is continuous in both time arguments.
Moreover, let
\begin{equation}
    \label{eqn:fm}
    \mathcal{L}_b(v^{\theta}) = \int_0^1\E\left[\left|v^{\theta}_{t, t}(I_t) - \partial_t{I}_t\right|^2\right]dt
\end{equation}
denote the flow matching loss on the diagonal $s=t$.
Then, the ideal flow map $X_{s, t}(x) = x + (t-s)v_{s, t}(x)$ corresponds to the unique minimizer over $v^{\theta}$ of each of the following objective functions.
\begin{enumerate}[label=\roman*.]
    \item The Lagrangian self-distillation loss,
    \begin{equation}
        \label{eqn:lsd}
       \lsd = \calL_b(v^{\theta}) + \int_0^1\int_0^t\E\left[\left|\partial_t X^{\theta}_{s, t}(I_s) - v^{\theta}_{t, t}(X^{\theta}_{s, t}(I_s))\right|^2\right]\ddd s \ddd t,
    \end{equation}
    \item The Eulerian self-distillation loss,
    \begin{equation}
        \label{eqn:esd}
       \esd = \calL_b(v^{\theta}) + \int_0^1 \int_0^t\E\left[\left|\partial_s X^{\theta}_{s, t}(I_s) + \nabla X^{\theta}_{s, t}(I_s)v^{\theta}_{s, s}(I_s)\right|^2\right]\ddd s\ddd t,
    \end{equation}
    \item The progressive self-distillation loss,
    \begin{equation}
       \label{eqn:psd} 
       \begin{aligned}
        &\psd = \calL_b(v^{\theta})\\
        & + \int_0^1\int_0^t\int_0^1\E\left[\left|v^{\theta}_{s, t}(I_s) - \gamma v^{\theta}_{s, u}(I_s) - (1-\gamma) v^{\theta}_{u, t}(X^{\theta}_{s, u}(I_s))\right|^2\right]\ddd s\ddd t \ddd\gamma,
       \end{aligned}
    \end{equation}
    where $u = (1-\gamma) s + \gamma t$ with $\gamma \in [0, 1]$.
\end{enumerate} 
\end{proposition}
\begin{proof}
For completeness, and for ease of the reader, we sketch the proof from~\citet{flowmaps}.
We first observe that $\calL_b(v^{\theta}) \geq \calL_b(v)$ where $v_t(x) = \E[\partial_t{I}_t | I_t = x]$ is the ideal flow.
From this, it follows that $\lsd \geq \calL_b(v)$, $\esd \geq \calL_b(v)$, and $\psd \geq \calL_b(v)$, because the second term in each loss is non-negative.
Moreover, by~\cref{prop:flow_map_char}, the ideal flow map satisfies the Lagrangian~\eqref{eqn:lagrangian} and the Eulerian~\eqref{eqn:eulerian}, so that the second term is zero for both $\lsd$ and $\esd$ at the true flow map.

It remains to show that the second loss in~\eqref{eqn:psd} imposes the semigroup condition.
To see this, we observe that we may write the semigroup condition as
\begin{equation}
\label{eqn:semigroup_derivation}
\begin{aligned}
    X_{s, t}(x) &= X_{u, t}(X_{s, u}(x)),\\ 
    \iff x + (t-s)v_{s, t}(x) &= X_{s, u}(x) + (t-u)v_{u, t}(X_{s, u}(x)),\\
    \iff x + (t-s)v_{s, t}(x) &= x + (u - s)v_{s, u}(x) + (t-u)v_{u, t}(X_{s, u}(x)),\\
    \iff (t-s)v_{s, t}(x) &= ((1 - \gamma)s + \gamma t - s)v_{s, u}(x)\\
    &\qquad + (t-(1 - \gamma)s - \gamma t))v_{u, t}(X_{s, u}(x)),\\
    \iff (t-s)v_{s, t}(x) &= \gamma (t-s)v_{s, u}(x) + (1-\gamma)(t-s)v_{u, t}(X_{s, u}(x)),\\
    \iff v_{s, t}(x) &= \gamma v_{s, u}(x) + (1-\gamma) v_{u, t}(X_{s, u}(x)).
\end{aligned}
\end{equation}
From the sequence of operations in~\eqref{eqn:semigroup_derivation}, we see that the second term in~\eqref{eqn:psd} penalises the square residual on a rescaled semigroup condition.
The minimiser therefore satisfies the semigroup condition, and by~\cref{prop:flow_map_char} is the ideal flow map.
\end{proof}

\xhdr{Stopgradients}
In practice, it is beneficial, when optimising over neural networks, to use the $\sg(\cdot)$ operator to control the flow of information.
This allows us to interpret the flow model $v^{\theta}_{t, t}$ as a ``teacher", which is used to provide a training signal for the map on the off diagonal $s \neq t$.
If the squared residuals presented in~\cref{prop:sd_euclid} are minimized directly, gradients can align the trained flow model with the untrained flow map, rather than vice-versa.
To this end, we provide some concrete recommendations inspired by the setting when a frozen pre-trained teacher model is available.
\begin{enumerate}
    \item For LSD, we recommend:
    \begin{equation}
        \label{eqn:lsd_sg}
        \lsd = \calL_b({\theta}) + \int_0^1\int_0^t\E\left[\left|\partial_t X^{\theta}_{s, t}(I_s) - \sg(v^{\theta}_{t, t}(X^{\theta}_{s, t}(I_s)))\right|^2\right]\ddd s\ddd t,
    \end{equation}
    \item For ESD, we recommend:
    \begin{equation}
        \label{eqn:esd_sg}
        \esd = \calL_b({\theta}) + \int_0^1 \int_0^t\E\left[\left|\partial_s X^{\theta}_{s, t}(I_s) + \sg(\nabla X^{\theta}_{s, t}(I_s)v^{\theta}_{s, s}(I_s))\right|^2\right]\ddd s\ddd t,
    \end{equation}
    \item For PSD, we recommend:
    \begin{equation}
       \label{eqn:psd_sg} 
       \begin{aligned}
        &\psd = \calL_b({\theta})\\
        & + \int_0^1\int_0^t\int_0^1\E\left[\left|v^{\theta}_{s, t}(I_s) - \sg(\gamma v^{\theta}_{s, u}(I_s) + (1-\gamma) v^{\theta}_{u, t}(X^{\theta}_{s, u}(I_s)))\right|^2\right]\ddd s\ddd t\ddd\gamma,
       \end{aligned}
    \end{equation}
    where $u = (1-\gamma) s + \gamma t$ with $\gamma \in [0, 1]$.
\end{enumerate}
The recommendations in~\eqref{eqn:lsd_sg} and~\eqref{eqn:psd_sg} are precisely what would arise given a pre-trained teacher.
The recommendation in~\eqref{eqn:esd_sg} is similar, but also places a $\sg(\cdot)$ on the spatial Jacobian of the model, which often improves optimization stability significantly.

\xhdr{Mean Flows and consistency training}
Given the choice of $\sg(\cdot)$ in~\eqref{eqn:esd_sg}, the resulting parameter gradient will be \textit{linear} in $v^{\theta}_{s, s}$.
We may then replace $v^{\theta}_{s, s}(I_s)$ by the Monte Carlo estimate $\partial_t{I}_s$ of the ideal flow $b_s(x) = \E[\partial_t{I}_s | I_s = x]$, using the tower property of the conditional expectation $\E\left[f(I_s)\partial_t{I}_s\right] = \E\left[f(I_s)\E\left[\partial_t{I}_s | I_s\right]\right] = \E\left[f(I_s)v_s(I_s)\right]$ for any function $f:\R^d\to\R^d$.
Without this choice of $\sg$, this replacement is not possible due to the quadratic term $\E[|\nabla X^{\theta}_{s, t}(I_s)\partial_t{I}_s|^2] \neq \E[|\nabla X^{\theta}_{s, t}(I_s)v_s(I_s)|^2]$ because of the nonlinearity in $|\cdot|^2$.
With this $\sg$, the quadratic term is a constant from the perspective of the gradient and hence this discrepancy can be ignored.

This trick is used by both consistency training~\citep{consistency, sCM} and Mean Flows~\citep{meanflows}; in fact, these algorithms can be recovered by expanding
\begin{equation}
\label{eqn:mean_flow_expansion}
\begin{aligned}
\partial_s X^{\theta}_{s, t}(x) &= -v^{\theta}_{s, t}(x) + (t-s)\partial_s v^{\theta}_{s, t}(x),\\
\nabla X^{\theta}_{s, t}(x) &= I + (t-s)\nabla v^{\theta}_{s, t}(x).
\end{aligned}
\end{equation}
Plugging~\eqref{eqn:mean_flow_expansion} into the Eulerian loss~\eqref{eqn:esd_sg} yields
\begin{equation*}
   \begin{aligned}
   &\calL({\theta}) = \calL_b({\theta})\\
   &+ \int_0^1 \int_0^t\E\left[\left|-v^{\theta}_{s, t}(I_s) + (t-s)\partial_s v^{\theta}_{s, t}(I_s) + \sg\left(v^{\theta}_{s, s}(I_s) + (t-s)\nabla v^{\theta}_{s, t}(I_s)v^{\theta}_{s, s}(I_s)\right)\right|^2\right]\ddd s \ddd t.
   \end{aligned}
\end{equation*}
Replacing $v^{\theta}_{s, s}(I_s)$ by the Monte Carlo estimate of the ideal flow as described above yields
\begin{equation}
   \label{eqn:mean_flow_2}  
   \begin{aligned}
   &\calL(v^{\theta}) = \calL_b(v^{\theta})\\
   &\qquad + \int_0^1 \int_0^t\E\left[\left|-v^{\theta}_{s, t}(I_s) + (t-s)\partial_s v^{\theta}_{s, t}(I_s) + \sg\left(\partial_s{I}_s + (t-s)\nabla v^{\theta}_{s, t}(I_s)\partial_s{I}_s\right)\right|^2\right]\ddd s \ddd t,
   \end{aligned}
\end{equation}
and then choosing to $\sg$ the $\partial_s v^{\theta}_{s, t}$ term as well as the spatial gradient yields the Mean Flow/Consistency Training objective,
\begin{equation}
   \label{eqn:mean_flow}  
   \begin{aligned}
   &\calL(v^{\theta}) = \calL_b(v^{\theta})\\
   &\qquad + \int_0^1 \int_0^t\E\left[\left|v^{\theta}_{s, t}(I_s) - \sg\left((t-s)\partial_s v^{\theta}_{s, t}(I_s) + \partial_s{I}_s + (t-s)\nabla v^{\theta}_{s, t}(I_s)\partial_s{I}_s\right)\right|^2\right]\ddd s \ddd t.
   \end{aligned}
\end{equation}

\section{Proofs}

\subsection{Generalised Flow Map Characterizations}
\label{app:legality_proof}

\begin{mdframed}[style=MyFrame2]
\propgfm*
\end{mdframed}

To show that the the propositions hold, let us first demonstrate the legality of the above claims, that the compared vectors indeed belong to the same tangent planes. For the semigroup condition, observe that $X_{s,t}(x) \in \M$ for any $s$, $t$ and $x$, and therefore is evidently coherent.
\begin{lemma}
\label{lemma:der_tangent}
For any $x \in \M$, $s,t \in [0,1]^2$, $\partial_t X_{s,t}(x) \in \T_{X_{s,t}(x)}\M$.
\end{lemma}
\begin{proof}[Proof of~\cref{lemma:der_tangent}]
Recall that, for any $x, s, t$, $X_{s,t}(x) = \exp_x((t-s)v_{s,t}(x))$, and, without loss of generality, suppose that $s \leq t$. (The case $t > s$ is perfectly symmetric.) For $s = 1$, $t=1$ and the proof is trivial, as $\vec 0 \in \T_x\M$. Let $s < 1$, and let $\tilde X_{u}(x) \coloneqq X_{s,(1-s)u + s}(x)$, which coincides with $X_{s,t}(x)$ for $(1-s)u + s = t \iff u = \frac{t-s}{1-s} =: \tilde u \in [0,1]$. Remarking that $((1-s)u + s) \in[0,1]$ for all $0 \leq u \leq 1$, we have that $\tilde X$ defines a geodesic between $\tilde X_0(x) = x$ and $\tilde X_1(x) = X_{s,1}(x)$. Now, let $u < 1$, as $u = 0 \implies s = t$, from where the proof is also trivial. Therefore, it follows that $\partial_u \tilde X_u (x) = \log_{ X_{s,u}(x)}( X_{s,1}(x))/(1-u)$, and considering the previous expression for $u = \tilde u$ finalises the proof from the definition of the logarithmic map.
\end{proof}
The above validates the Lagrangian condition. Let us also justify the Eulerian condition: by~\cref{lemma:der_tangent}, $\partial_sX_{s,t}(x_s) \in \T_{X_{s,t}(x_s)}\M$, and let us now prove that $\ddd (X_{s,t})_x[v_s(x)]$ is on the same tangent space.
\begin{lemma}
For any $x \in \M$, $(s,t) \in [0,1]^2$, $\ddd (X_{s,t})_x[v_s(x)] \in \T_{X_{s, t}(x)}\M$.
\label{lemma:euler-just}
\end{lemma}
\begin{proof}[Proof of~\cref{lemma:euler-just}]
This is true by definition. To see this, let $\gamma:[0,1]\to\M, \tau\mapsto \gamma(\tau)$ be a curve, such that $\gamma(0) = x$ and $\partial_\tau \gamma = v_s$. We therefore have that $X_{s,t} \circ \gamma:[0,1]\to\M$, and therefore $\partial_\tau(X_{s,t}\circ\gamma)(0) \in \T_{X_{s, t}(\gamma(0))}\M = \T_{X_{s, t}(x)}\M$. Knowing that $(X_{s,t})_x v_s(x) = \partial_\tau (X_{s,t}\circ\gamma)(0)$ concludes the proof.
\end{proof}
We can now prove the main proposition.
\begin{proof}[Proof of~\cref{prop:flow_map_char}]
Let us first prove the Lagrangian condition. It comes naturally as
\begin{equation}
    X_{s,t}(x_s) = x_t \implies \partial_t X_{s,t}(x_s) = \partial_t x_t = v_t(x),
\end{equation}
which is true by definition, since $(x_t)_{t\in[0,1]}$ is a solution of \cref{eq:prob_flow_on_manifolds}. As for the Eulerian condition, we note too that, for any $x \in \M$, $X_{s,t}(X_{t,s}(x)) = x$. (So, it is also invertible.) Taking the derivative through $s$,
\begin{gather}
    \partial_s X_{s,t}(X_{t,s}(x)) = \partial_s x\\
    \implies \partial_s X_{s,t}(X_{t,s}(x)) + (\ddd X_{s,t})_{X_{t,s}(x)}[\partial_s X_{t,s}(x)] = 0,
\end{gather}
where the last line is due to the chain rule on manifolds. Moreover, by the probability flow ODE, we have that $\partial_sX_{t,s}(x) = v_s(X_{t,s}(x))$, and, letting $y = X_{t,s}(x)$, we have that
\begin{equation}
    \partial_s X_{s,t}(y) + (\ddd X_{s,t})_y[v_s(y)] = 0.
\end{equation}
Since $X_{s,t}$ is invertible for any $s$ and $t$ and defined on $\M$, it is bijective, and thus we can freely rename $y$ into $x$ to conclude that the Eulerian condition holds. Finally, the semigroup condition is trivially true from the definition: $(X_{u,t}\circ X_{s, u})(x_s) = x_t = X_{s,t}(x_s)$.
\end{proof}

\subsection{Proof for~\cref{lemma:generalisedtangentcondition}}
\label{app:lemma_generalised_tangent_condition_proof}

\begin{mdframed}[style=MyFrame2]
\generalisedtangentcondition*
\end{mdframed}
\begin{proof}{Proof of~\cref{lemma:generalisedtangentcondition}.}
This property follows from the Lagrangian condition. We have that $\partial_t X_{s,t}(x_s) = v_t(X_{s,t}(x_s))$. Taking the limit on both sides, we find that
\begin{gather}
    \lim_{s\to t}\partial_t X_{s,t}(x_s) = \lim_{s\to t}v_t(X_{s,t}(x_s))\\
    \implies \lim_{s\to t}\partial_t X_{s,t}(x_s) = v_t\left(\lim_{s\to t}X_{s,t}(x_s)\right) = v_t(x_t),
\end{gather}
where the last line comes from the continuity of the limit.
\end{proof}
\begin{lemma}
For any $(s,t) \in [0,1]^2$ and $x \in \M$, $\lim_{s\to t}\partial_t X_{s,t}(x) = v_{t,t}(x)$.
\label{lemma:der_flowmap}
\end{lemma}
\begin{proof}
Following elementary Riemannian geometry, we find that
\begin{equation}
    \partial_t X_{s,t}(x) = \partial_t \exp_x((t-s)v_{s,t}(x)) = \ddd(\exp_x)_{(t-s)v_{s,t}(x)}(v_{s,t}(x) + (t-s)\partial_t v_{s,t}(x)).
\end{equation}
Knowing that $\ddd(\exp_x)_{\vec 0} = \mathrm{Id}_{\T_x\M}$, and taking the limit, we find the desired conclusion.
\end{proof}

\subsection{Generalised Self-Distillation Losses}
\label{app:proofs_generalised_losses}
\cut{
\begin{lemma}
\label{lemma:der_vec_tangent}
For any $x \in \M$, $s,t\in[0,1]^2$, $\partial_t v_{s,t}(x) \in \T_x\M$.
\end{lemma}
\begin{proof}[Proof of~\cref{lemma:der_vec_tangent}]
Similarly to the above, suppose, without loss of generality, that $s \leq t$. Knowing that $\T_x\M$ is a vector space for any $x \in \M$, let $\T_x\M = \mathrm{span} \{e_1, \ldots, e_d\}$ for $d = \dim \M$ be a basis thereof. Let $x \in \M$, and for any $0 \leq t \leq 1$ let $\tilde v(t) = v_{s,t}(x)$. Since $\tilde v(\cdot) = v_{s,\cdot}(x) \in \T_x\M$, we can write that $\tilde v(t) = \sum_{i=1}^d a_i(t) e_i$ for some functions $a_1,\ldots,a_d\in\mathrm{C}^1([0,1], \R)$. Evidently, we then have that $\partial_t \tilde v(t) = \sum_{i=1}^d \partial_t a_i(t) e_i \in \T_x \M$ as it is also a linear combination of the basis vectors.
\end{proof}
}
\begin{mdframed}[style=MyFrame2]
\propglsd*
\end{mdframed}
\begin{proof}
The loss is zero if and only if $\partial_t X_{s,t}^\theta(I_s) = v_{t,t}^\theta(X_{s,t}^\theta(I_s))$ almost everywhere. We can conclude by~\cref{prop:gfm_characterization}.
\end{proof}

\begin{mdframed}[style=MyFrame2]
\propgesd*
\end{mdframed}
\begin{proof}
The loss is zero if and only if $X_{s,t}^\theta(x_s) - \ddd(X_{s,t}^\theta)_{x_s}[v_s^\theta(x_s)] = 0$ almost everywhere. We can conclude by~\cref{prop:gfm_characterization}.
\end{proof}

\begin{mdframed}[style=MyFrame2]
\propgpsd*
\end{mdframed}
\begin{proof}
The loss is zero if and only if $X_{u,t}^\theta(X_{s,u}^\theta(x)) = X_{s,t}^\theta(x)$ almost everywhere. We can conclude by~\cref{prop:gfm_characterization}.
\end{proof}

\subsection{Connections to existing methods}
\label{app:connection_existing_methods}
\subsubsection{Generalised Mean Flows}
Generalising Mean Flows directly is non-trivial as it requires defining the integral of a vector field on a manifold properly, which would involve parallel transport and therefore derivatives thereof. Also, Mean Flows operate on the vector field level as opposed to Flow Map Matching which operates on the level of the flow map. It is difficult to go from one level to another directly, as it will involve non-trivial curvature terms. Instead, we propose to heuristically follow our derivations in~\cref{app:flow_map_losses_original}, in the ``stopgradients'' section. Indeed, we can see that our loss involves the instantaneous vector field of the modelled flow map, $v^\theta_{t,t}$, as opposed to the ideal flow $\partial_t I_t$; hence the use of the Levi-Civita connection along $v_s(I_s)$ instead of the differential evaluated at $v_{s,s}^\theta$:
\begin{equation}
\hat \gL_{\mathrm{G\text{-}MF}}({\theta}) = \E_{t,s, (x_0,x_1)}\left[\left\lVert v_{s,t}^\theta(I_s) - \sg\left(\partial_s I_s - (t - s) \nabla_{v_s(I_s)}v^\theta_{s,t}(I_s)\right) \right\rVert_g^2\right],
\end{equation}
which indeed recovers the Euclidean case as a special case. The Levi-Civita connection, $\nabla:\mathfrak{X}(\M)\times \mathfrak{X}(\M) \to\mathfrak{X}(\M)$, $(v, X)\mapsto \nabla_v X$, is the unique torsion-free, bilinear, metric compatible connection on $(\M, g)$ that respects the Leibniz rule in $X$, and it defines a notion of covariant derivative for the vector fields on $(\M, g)$.

\cut{ we follow the steps in~\cref{app:flow_map_losses_original}. We start by expanding the term:
\begin{align}
    \partial_s X_{s,t}^\theta(x) &= \ddd(\exp_x)_{(t-s)v_{s,t}^\theta(x)}(-v_{s,t}^\theta(x) + (t-s)\partial_s v_{s,t}^\theta(x)).
\end{align}}

\subsubsection{Generalised Shortcut Models}
Shortcut models~\citep{shortcut} are exactly trained to enforce the semigroup property but on a discrete time grid. Expressed in our notation, the loss amounts to
\begin{equation}
\gL(\theta) = \gL_\mathrm{FM}(\theta) + \E_{s,d,(x_0,x_1)}\left\lVert v_{s,s+2d}(x_s) - \sg\left(v_{s+d,s+2d}(x_s + dv_{s,s+d}(x_{s}) \right)\right\rVert_2^2.
\end{equation}
Indeed, $d$ can be seen as the time difference $t-s$, and, letting $d$ be uniformly distributed, we recover exactly Euclidean Progressive self-distillation, as noted in~\citet{boffi2024flow}. The Riemannian case is strictly analogous, where the + operation is replaced by the exponential map, and the appropriate manifold distance is used.

\section{Additional Experimental Details}
\label{app:additional_experimental_details}
{
\subsection{The relevance of NLL}
\label{app:nll_relevance}
The NLL is a standard generative modelling benchmark for generative models~\citep{rfm,flowmatching,song2020score}. It is true, however, that in our context this metric does not evaluate directly the quality of our learnt GFMs. It is still relevant in at least two important ways:
\begin{enumerate}
    \item It shows that the instantaneous vector field is still well-learnt, despite the self-distillation loss, which could have impacted the training dynamics on this part of the loss. We show that it does not. So, overall, our objective did not “sacrifice” the instantaneous vector field, and the sum of the self-distillation and flow matching losses are not inherently “incompatible”.

    \item The flow map and its learnt instantaneous vector field are approximations of one another. Therefore, while the likelihood of one is not \emph{exactly} equal to that of the other, the difference between the two being typically very low (as measured by the self-distillation losses, down to $10^{-5}$ in MSE), the instantaneous vector field allows us to compute a good approximation of the likelihood induced by the flow map.
\end{enumerate}
}

\subsection{Empirical MMD calculation}
For any distributions $p$ and $q$ with support $(\M, g)$, and with respective independent samples $(p_i)_{1 \leq i \leq n}$ and $(q_i)_{1 \leq i \leq n}$,
\begin{equation}
\label{eq:mmd}
    \widehat \MMD(p, q) \coloneqq \frac 1 {n^2} \sum_{i,j=1}^n \exp\left(-\kappa d_g^2(p_i, p_j) \right) + \exp\left(-\kappa d_g^2(q_i, q_j) \right) - 2 \exp\left(-\kappa d_g^2(p_i, q_j) \right).
\end{equation}
\looseness=-1
We choose $n$ to be equal to the size of the test set.

\subsection{Additional results}
{\cref{fig:other_earth} includes additional (all) qualitative results on the Earth datasets.}

\begin{figure}[htp]
  \centering
  \caption{Plots of densities for the various datasets and all compared methods. Depicted in red are the test-set samples. Datasets from left to right: volcano, earthquake, flood, fire.}
  \vspace{-5pt}
  \label{fig:other_earth}
  \resizebox{\linewidth}{!}{
  \begin{tabular}{m{0.1\linewidth}m{0.9\linewidth}}
    \ourspsd
      & \includegraphics[width=\linewidth]{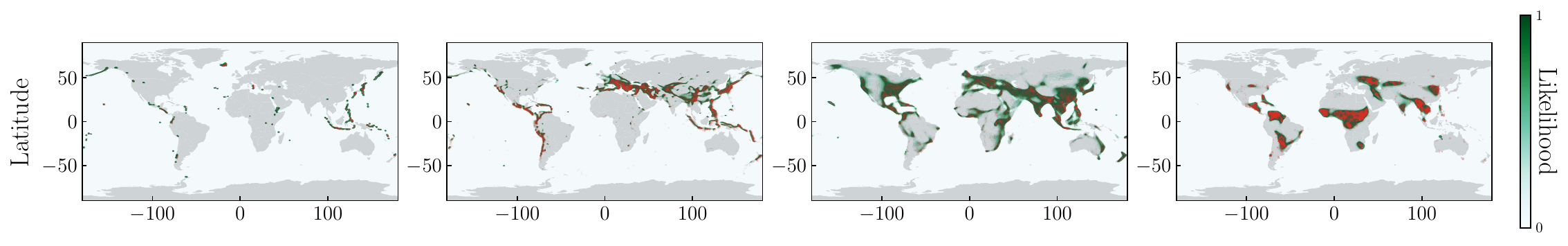} \\
    \oursesd
      & \includegraphics[width=\linewidth]{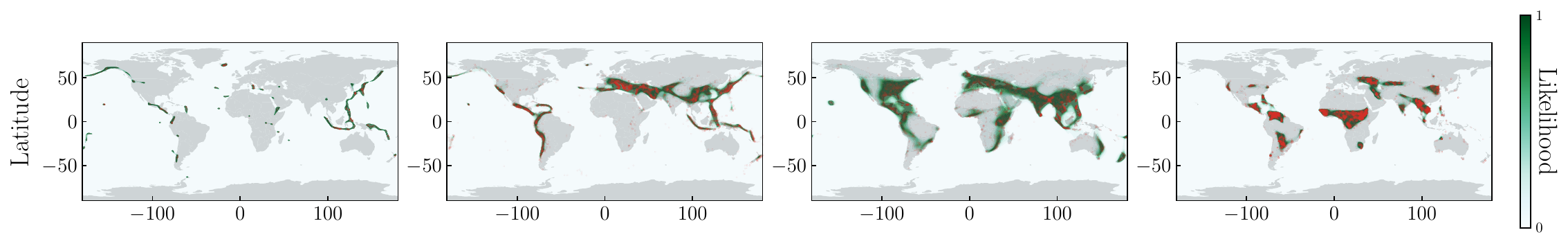} \\
    \raisebox{1.9ex}{\oursmf}
      & \includegraphics[width=\linewidth]{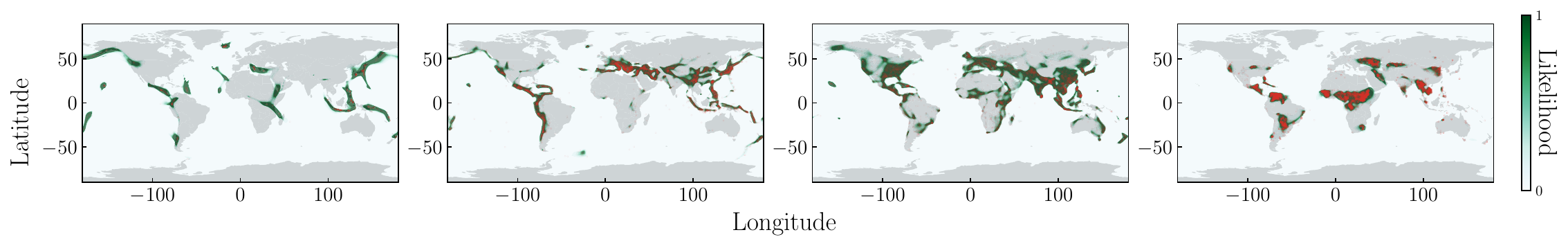}
  \end{tabular}
  }
\end{figure}

{
\subsection{The computational cost of self-distillation}
\label{app:compute}
We include, here,~\cref{tab:runtime_training} exhibiting the differences in training run-times (wall clock time) between our GFM objectives and plain RFM; the numbers reported have been gathered on the Volcano dataset, which ran the largest model (and therefore is arguably most relevant; about $45$M parameters). The gradation is rather intuitive: RFM has a simple flow matching loss; G-ESD computes a spatial Jacobian of a large mode; G-LSD only computes a time derivative of the same model; and G-PSD requires two additional forward passes without any gradients, compared to RFM.
\begin{table}[tbp]
    \centering
    \resizebox{\linewidth}{!}{
    \begin{tabular}{c|c|c|c}
        \toprule
         \textbf{RFM} & \textbf{G-ESD} & \textbf{G-LSD} & \textbf{G-PSD} \\
         \midrule
         $0.60\pm0.07$ epochs/min & $0.29\pm0.01$ epochs/min & $0.36\pm 0.02$ epochs/min& $0.43\pm0.04$ epochs/min\\
         \bottomrule
    \end{tabular}
    }
    \caption{Wall clock  throughput time of our proposed methods on the volcano dataset (higher is faster).}
    
    \label{tab:runtime_training}
\end{table}
}